\documentclass{article} %

\usepackage{amsthm} 
\usepackage{mathtools}
\usepackage{amsfonts}
\usepackage{comment}

\usepackage[hidelinks]{hyperref}

\usepackage{booktabs}

\usepackage{algorithm}
\usepackage{algorithmic}

\usepackage{threeparttable}

\usepackage{makecell}
\usepackage{dsfont}
\usepackage{wrapfig}
\usepackage{subcaption}
\usepackage{changepage}

\newtheorem{definition}{Definition}
\newtheorem{theorem}{Theorem}
\newtheorem{proposition}{Proposition}
\newtheorem{lemma}{Lemma}
\newtheorem{corollary}{Corollary}

\theoremstyle{remark}
\newtheorem{remark}{Remark}
\newtheorem{assumption}{Assumption}

\newcommand{\Lupper}{\mathfrak{L}_{\textrm{\textup{upper}}}}
\newcommand{\Llower}{\mathfrak{L}_{\textrm{\textup{lower}}}}
\newcommand{\gradw}{\nabla_{\omega}}

\newcommand{\wb}{\omega}
\newcommand{\norm}[1]{\left\lVert#1\right\rVert}

\usepackage{iclr2026_conference,times}

\usepackage{amsmath,amsfonts,bm}

\def\eqref#1{equation~\ref{#1}}

\def\1{\bm{1}}

\DeclareMathAlphabet{\mathsfit}{\encodingdefault}{\sfdefault}{m}{sl}
\SetMathAlphabet{\mathsfit}{bold}{\encodingdefault}{\sfdefault}{bx}{n}

\newcommand{\R}{\mathbb{R}}

\DeclareMathOperator*{\argmin}{arg\,min}

\usepackage[hidelinks]{hyperref}
\usepackage{url}

\title{Iterative Training of Physics-Informed \\ Neural Networks with Fourier-enhanced \\ Features}

\author{Yulun Wu, Miguel Aguiar, Karl H.~Johansson \& Matthieu Barreau \\
Division of Decision and Control Systems\\
Digital Futures and KTH Royal Institute of Technology\\
Stockholm, Sweden \\
\texttt{\{yulunw,aguiar,kallej,barreau\}@kth.se}
}

\iclrfinalcopy %

\begin{document}

\maketitle

\begin{abstract}
Spectral bias, the tendency of neural networks to learn low-frequency features first, is a well-known issue with many training algorithms for physics-informed neural networks (PINNs).
To overcome this issue, we propose IFeF-PINN, an algorithm for iterative training of PINNs with Fourier-enhanced features. The key idea is to enrich the latent space using high-frequency components through random Fourier features. This creates a two-stage training problem: (i) estimate a basis in the feature space, and (ii) perform regression to determine the coefficients of the enhanced basis functions. For an underlying linear model, it is shown that the latter problem is convex, and we prove that the iterative training scheme converges. Furthermore, we empirically establish that random Fourier features enhance the expressive capacity of the network, enabling accurate approximation of high-frequency PDEs.
Through extensive numerical evaluation on classical benchmark problems, the superior performance of our method over state-of-the-art algorithms is shown, and the improved approximation across the frequency domain is illustrated.

\end{abstract}

\section{Introduction}
\label{sec:intro}

Capturing high-frequency behavior is central to modeling complex phenomena such as wave propagation, turbulence, and quantum dynamics. Traditional numerical methods, including spectral approaches \citep{boyd2001chebyshev}, multiscale schemes \citep{weinan2003heterognous}, and oscillatory quadrature \citep{iserles2005efficient}, have achieved notable success but often require problem-specific adaptations or become prohibitively costly in complex or high-dimensional settings.

There is a need for new approximation strategies that capture high-frequency behavior without sacrificing stability or tractability. Deep-learning surrogates of differential equations are a promising alternative, such as Physics-Informed Neural Networks (PINNs), which offer a grid-free alternative by combining data and physical models within a neural network framework \citep{raissi2017physics}. This paradigm has shown strong performance in solving partial differential equations (PDEs) and inferring hidden dynamics, benefiting 
adaptability to complex geometries \citep{costabal2024delta}, and high-dimensional scalability \citep{hu2024tackling}. Related approaches such as Fourier Neural Operators \citep{li2021fourier} and DeepONet \citep{lu2021learning} further expand its reach. Despite these advances, PINN methods remain limited by \emph{spectral bias}---the tendency of neural networks to learn low-frequency components first---which hinders accurate recovery of oscillatory solutions \citep{rahaman2019spectral,xu2025understanding,lin2021operator,qin2024toward}.

Several strategies have been proposed to mitigate spectral bias, including weight balancing \citep{wang2021understanding,krishnapriyan2021characterizing}, resampling \citep{lau2024pinnacle,tang2024adversarial,song2025rl}, and curriculum or architecture-based approaches \citep{sirignano2018dgm,waheed2022kronecker,chai2024overcoming,mustajab2024physics,eshkofti2025vanishing,wang2024multi}. Table~\ref{tab:existing_methods} summarizes some of the most representative approaches. While effective in certain cases, these methods remain tied to single-level optimization frameworks, where feature learning and coefficient fitting are intertwined in neural networks, limiting both robustness and theoretical guarantees.

To address this gap, we draw inspiration from classical numerical PDE solvers, which approximate solutions using basis functions, and propose a novel neural network architecture and a tailored training algorithm. 
The key idea is to create a feed-forward neural network with three components, as illustrated in Figure~\ref{fig:general_figure}. First, the hidden layers $h_{\omega}$ generate a nominal basis in the latent functional space. 
Next, this basis is extended to $\psi_D$ through random Fourier features (RFF, introduced by \cite{rahimi2007random}), which may include potentially higher-frequency elements, to span a larger latent space.
Finally, the last linear layer performs regression on these extended basis vectors. The first and last blocks can be optimized separately, resulting in a two-stage iterative scheme alternating between latent basis construction and regression on output coefficients. A major feature of this framework, related to extreme learning machines \citep{dwivedi2020physics}, is that for linear differential equations, the regression stage is convex and achieves asymptotic global optimality. Unlike existing approaches, our method enriches the latent space representation, enabling systematic capture of high-frequency dynamics while leveraging the strengths of established PINN frameworks.  

\begin{figure}
    \centering
        \includegraphics[width=0.66\linewidth]{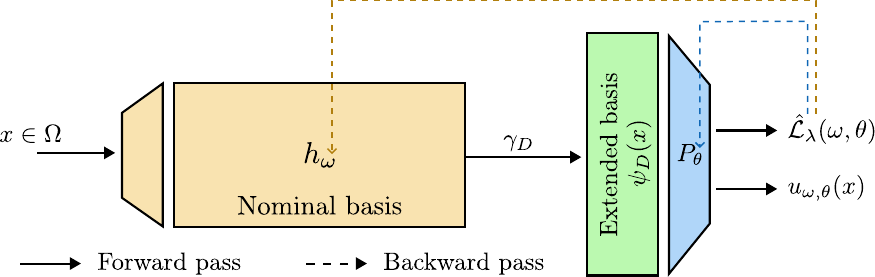}
    \caption{Architecture of IFeF-PINN. The first part (in yellow) generates the nominal basis vectors, which are then extended via $\gamma_D$ generating random Fourier features $\psi_D$ (in green), and a linear combination of the extended basis (in blue) forms the approximated solution $u_{\omega,\theta}$.}
    \label{fig:general_figure}
    \vspace*{-0.3cm} %
\end{figure}

\begin{table}
\centering
\footnotesize
\setlength{\tabcolsep}{4pt} %
\renewcommand{\arraystretch}{1.5}
\caption{Representative methods for approximating solutions to PDE, highlighting application domain, key idea, high-frequency handling (HF), limitations, and optimality.}
\vspace*{-0.2cm}
\label{tab:existing_methods}
\begin{tabular}{p{3cm}|p{1.5cm}|p{3cm}|c|p{4.4cm}}
\toprule
\textbf{Method} & \textbf{Domain} & \textbf{Key Idea} & \textbf{HF} & \textbf{Limitations / Optimality} \\ \midrule
\cite{boyd2001chebyshev,iserles2005efficient} & Linear & Global basis functions (Fourier, Chebyshev) & +++ & Requires regular domains; \textbf{global optimum} \\ 
\cite{weinan2003heterognous} & Multiscale & Separate scales and compute effective dynamics & ++ & Needs clear scale separation, problem-specific; \textbf{local optimum} \\ 
\cite{raissi2017physics} & Generic & NN minimizing physics + data loss & - & Struggles with high-frequency components; \textbf{local optimum} \\ 
\cite{li2021fourier,lu2021learning} & Operator & Learn mapping in Fourier / function space & + & Problem-specific, may require large networks; \textbf{local optimum} \\
\cite{chai2024overcoming,ZhaoEtAl24} & Multiscale & Network architecture or training strategy & ++ & Problem-specific, not robust; \textbf{local optimum} \\ %
\cite{lau2024pinnacle,tang2024adversarial,song2025rl}& General & Adaptive resampling & ++ & Computationally expensive, no convergence guarantees; \textbf{local optimum} \\
\textbf{IFeF-PINN} (this work) & Generic & Iterative training  with extended basis via Fourier features & +++ & Not adapted to resampling, high memory footprint; \textbf{Global optimum} (for linear PDEs) \\ 
\bottomrule
\end{tabular}
\vspace*{-0.75cm} %
\end{table}

In this paper, we propose Iterative PINNs with Fourier-Enhanced Features (IFeF-PINN), a novel iterative two-stage training algorithm that mitigates the spectral bias of PINNs in high-frequency problems while maintaining accurate approximation on standard benchmark PDEs.
Our contributions are threefold: (i) we introduce a flexible building block that augments existing PINNs architectures with improved high-frequency estimation and demonstrate its universal approximation capabilities; (ii) we propose an iterative two-stage training algorithm and prove its convergence properties; and (iii) we validate the approach through extensive simulations on benchmark problems, showing substantial improvements over existing methods.

\section{Background}
\label{sec:background}

\subsection{Physics-Informed Neural Networks}
\label{sec:PINN}

PINNs is a deep learning framework that integrates PDEs into the neural network training via the loss function, enabling data-driven learning with physical constraints \citep{raissi2017physics,karniadakis2021physics}.

Generally, for $n > 0$, let $\Omega \subset \R^n$ be a bounded domain and $\mathcal W$ an appropriate Sobolev space of functions from $\Omega$ to $\R$, we consider linear PDEs of the form
\begin{equation}
\begin{aligned}
\label{equ:general_pde}
&\mathfrak{F}[u](x) = f(x), \quad x \in \Omega, \\
& \mathfrak{B}[u](s) = g(s), \quad s \in \Gamma \subseteq \partial \Omega,
\end{aligned}
\end{equation}
where $u \in \mathcal W$ is the solution, $\mathfrak{F}:\mathcal W \to \mathcal{L}^2(\R^n, \R)$ is the linear differential operator, $f \in \mathcal{L}^2(\Omega, \R)$ is the source term, $\mathfrak{B}:\mathcal W \to \mathcal{Y}(\Gamma)$ is the linear boundary/initial operator, $g \in \mathcal{Y}(\Gamma)$ specifies the boundary/initial conditions, where $\mathcal{Y}(\Gamma)$ denotes the appropriate trace space.
We assume that this problem is well-posed and therefore has a unique solution in $\mathcal W$.

The objective of PINNs is to approximate the solution $u$ with a feedforward neural network $u_\omega$, where $\omega$ denotes the network parameters. \cite{shin2020on} and~\cite{sirignano2018dgm} analyzed consistency in weak formulations under suitable assumptions, motivating the following continuum loss:
\begin{equation}
    \label{eq:pinn_loss}
    \mathfrak{L}_{\lambda}(u_\omega) = \frac{1}{|\Gamma|}\int_{\Gamma} \| g(s) - \mathfrak{B}[u_\omega](s) \|^2 ds + \frac{\lambda}{|\Omega|} \int_{\Omega} \| \mathfrak{F}[u_\omega](x) \|^2 dx,
\end{equation}
with $\lambda > 0$ where, for $A$ a bounded set, $|A|$ denotes its measure. However, this version is not numerically tractable and, in practice, we use the Monte Carlo approximation
\begin{equation}
\label{eq:pinn_sampled_loss}
\hat{\mathfrak{L}}_{\lambda}(u_\omega) = \frac{1}{N_u} \sum_{i=1}^{N_u} \| g(x_u^i) - \mathfrak{B}[u_\omega](x_u^i) \|^2 + \frac{\lambda}{N_f} \sum_{i=1}^{N_f} \|\mathfrak{F}[u_\omega](x_f^i)\|^2,
\end{equation}
where $\{x_u^i\}_{i=1,\dots,N_u}$ and $\{x_f^i\}_{i=1,\dots,N_f}$ are uniformly sampled on $\Gamma$ and $\Omega$, respectively.
Finally, the optimal parameters are found as $\omega^* = \arg\min_{\omega} \hat{\mathfrak{L}}_{\lambda}(u_\omega)$.

\subsection{Random Fourier features}

In this work, we use random Fourier features (RFFs) introduced by~\cite{rahimi2007random} to include high-frequency terms.
Grounded on Bochner's theorem, RFF provides a way to explicitly construct a feature map that approximates a stationary kernel, enabling the scaling of kernel methods to large datasets.

RFF has been used by \cite{tancik2020fourier} to tackle spectral bias. 
The novelty is to extend the input to the neural network using the RFF mapping
\begin{equation}
    \label{eq:rff_mapping}
    \gamma_D(x) = \frac{1}{\sqrt{D}} \begin{bmatrix} \cos(2\pi \mathbf{B}_D x) \\
            \sin(2\pi \mathbf{B}_D x) \end{bmatrix} \in \mathbb{R}^{2D},
\end{equation}
where the entries of the matrix $\mathbf{B}_D \in \mathbb{R}^{D \times n}$ are sampled from a given symmetric distribution. \cite{wang2021eigenvector} adapted this method to PINNs by using $u_\omega$ from the previous section with $2D$ inputs, so that the neural network becomes $u_\omega \circ \gamma_D$. This new architecture can learn to approximate the solution from the enriched inputs. 

\section{Proposed Method}
\label{sec:method}

We leverage the PINNs and RFFs in a novel way. Note first that the PINN training process couples two roles within a single nonconvex objective: (i) hidden layers $h_\omega$ learn a nonlinear feature basis, and (ii) a linear regression operator $P_\theta: h_\omega \mapsto h_\omega^{\top} \theta$ finds the optimal projection coefficients $\theta$ of the approximated solution onto the feature basis, thereby minimizing the loss $\hat{\mathfrak{L}}_\lambda$. This coupling leads to PINN pathologies, where gradients from interior residuals can dominate and suppress boundary terms, and spectral bias drives low-frequency learning first, leaving oscillatory components underfit and slowing convergence on high-frequency modes \citep{wang2021eigenvector,wang2022and}.

To overcome this coupling issue, we approximate the solution $u$ to the PDEs in (\ref{equ:general_pde}) as a linear combination of basis functions.
We thus consider the two problems in isolation: basis generation, which we will denote as the upper-level problem, and linear regression on the basis functions, which we will refer to as the lower-level problem.

\subsection{The upper-level problem: Basis function generation}

The initial step for the basis generation is to follow the classical PINN methodology and train a standard feed-forward neural network with parameters $(\omega, W)$, denoted by
\begin{equation*}
    \tilde{u}_{\omega,W}(x) = W h_\omega(x), \quad \quad x \in \Omega,
\end{equation*}
to minimize $\omega,W \mapsto \hat{\mathfrak{L}}_{\lambda}(\tilde{u}_{\omega,W})$. This is typically accomplished using a gradient-descent numerical scheme, such as ADAM~\citep{kingma2014adam}, or a more complex second-order solver, like L-BFGS~\citep{liu1989limited}. Then, the neural network $h_\omega: \R^n \to \R^p$ generates a basis $h_\omega \in \mathcal{C}(\R, \R^p)$ of the latent space while $W$ is the projection operator.
This initial step serves as a warm-up for the upper-level problem. Note that $\tilde{u}_{\omega,W}$ most likely contains only the low-frequency components of the original solution.
Therefore, the surrogate $\tilde{u}_{\omega,W}$ might be an aliased or steady-state solution of the PDE, and the fit at the boundary points might be poor.

In our approach, the strategy is to apply an RFF mapping to the last hidden layer features $h_\omega$. This upgrades the implicit linear kernel on $h_{\wb}$ to a stationary kernel, such as a radial basis function, in the adaptive feature space. Since $\tilde{u}$ is probably a distorted version of the real solution $u$, the RFF extension might bring higher frequency signals that mitigate the spectral bias.

Concretely, we define $\psi_D(x) = \gamma_D\left(h_{\omega}(x)\right) = \frac{1}{\sqrt{D}} \begin{bmatrix} \cos(2\pi \mathbf{B}_D h_\wb(x)) \\
            \sin(2\pi \mathbf{B}_D h_\wb(x)) \end{bmatrix}$
where $\mathbf{B}_D \in \R^{D\times p}$ is a constant matrix with entries sampled i.i.d. from $\mathcal{N}(0,\sigma^2)$.

\subsection{The lower-level problem: Linear regression}
The linear output layer over $h_{\wb}$ induces a dot-product kernel in feature space, which can limit expressivity and exacerbate spectral bias toward low frequencies. Applying RFF to $h_{\wb}$ equips the adaptive features with a stationary kernel without adding trainable parameters, injecting high-frequency components via random projections. Formally speaking, an approximate solution to the PDE in (\ref{equ:general_pde}) with $\theta\in\R^{2D}$ becomes 
\begin{equation}
\label{equ:u_basis_approx}
      u_{\wb,\theta}(x) = \psi_D(x)^\top \theta, \quad x \in \Omega.
\end{equation}

As we show in Appendix B,
since the operators $\mathfrak{F}$ and $\mathfrak{B}$ are linear,
the loss function $\hat{\mathfrak{L}}_{\lambda}(u_{\omega, \theta})$ is quadratic in $\theta$:
\begin{equation}
\label{equ:Qcb_loss}
\Llower(\theta \ | \ \omega) := \hat{\mathfrak{L}}_{\lambda}(u_{\omega,\theta}) = \tfrac{1}{2}\,\theta^\top Q(\omega) \theta + c(\omega)^\top \theta + b,
\end{equation}
where $Q$ and $c$ collect boundary and interior residual terms.

\begin{proposition}
    \label{prop:QP}
    Assume that $\lambda > 0$ and that the rank condition (3) 
    from Appendix B.1
    is verified. Then $Q$ is positive definite and there is a unique solution to $\argmin_\theta \Llower(\theta \ | \ \omega) = - Q^{-1}(\omega) c(\omega)$.
\end{proposition}

The proof is given in Appendix B.1.
The application of the RFF mapping in the last hidden layer enables the generation of an arbitrary number of basis functions $\psi_D$ independently of the network's width on which we can leverage quadratic programming to get the unique optimal solution. This would otherwise not be possible because constrained by the basis dimension.

\subsection{The global bi-level problem}

Combining the results from the two previous subsections, we get the following formulation that decouples basis learning (upper-level) from linear regression (lower-level):
\begin{equation}
    \label{eqp:bi_level}
    \begin{aligned}
    \wb^\star(\theta) &= \arg\min_{\wb}\hat{\mathfrak{L}}_{\lambda}(u_{\wb,\theta}) := \arg\min_{\wb}\ \Lupper(\wb \ | \ \theta),\\
    \theta^\star(\wb) &= \arg\min_{\theta}\ \hat{\mathfrak{L}}_{\lambda}(u_{\wb, \theta}) := \arg\min_{\theta}\ \Llower(\theta \ | \ \omega).
    \end{aligned}
\end{equation}

The classical bi-level optimization framework \citep{bard1991some} proposes the following three-step numerical method: (i) sample $w_0$, $\theta_0$ randomly; (ii) solve the upper-level problem $\omega^+ = \omega^\star(\theta_0)$; (iii) solve the lower-level problem $\theta^+ = \theta^{\star}(\omega^+)$. The final parameters $(\omega^+, \theta^+)$ are the optimal solutions to the bi-level optimization.

\begin{wrapfigure}{R}{0.45\textwidth}
    \begin{minipage}{0.45\textwidth}
        \vspace*{-0.85cm}
        \begin{algorithm}[H]
            \caption{IFeF-PINN for linear PDEs}
            \label{alg:blo}
            \begin{adjustwidth}{-1em}{}
            \begin{algorithmic}
            \STATE \textbf{Initialize} network parameter $w_0, \theta_0$ and $B$ 
            \FOR{$k$ from $0$ \TO $N_{epoch}$}
                \STATE Formulate extended RFF basis $\psi_D$
                \STATE \textbf{Lower update:} $\theta_{k+1}  = - Q(\omega_{k})^{-1}  c(\omega_{k})$
                \vspace*{-0.15cm}
                \STATE \textbf{Upper update:} 
                \vspace*{-0.15cm}
                \[
                    \omega_{k+1} = \omega_k - \eta\nabla_\omega \Lupper(\wb_k \ | \ \theta_{k+1})
                \]
                \vspace*{-0.4cm}
            \ENDFOR
            \RETURN $\omega_{N_{epoch}}$, $\theta^\star(\omega_{N_{epoch}})$
        \end{algorithmic}
        \end{adjustwidth}
        \end{algorithm}
        \vspace*{-0.8cm}
    \end{minipage}
\end{wrapfigure}

However, this approach does not consider a warm start and is not particularly adapted to a learning problem. For better approximation capabilities, we propose an iterative scheme.
We warm start using a vanilla PINN pre-training to get an initial value $\omega_0$ for the weights of the basis generator. Then we compute $\theta_{i+1} = \theta^*(w_i)$ before performing a one-step gradient-descent on $\omega_i$ to minimize $\Lupper(\omega_i \ | \ \theta_{i+1})$ to get $\omega_{i+1}$. This leads to Algorithm~\ref{alg:blo}. The convergence of this numerical scheme and the approximation capabilities of the new neural network architecture are studied in the next section.  

\begin{remark}[Relation to deep kernel learning]
    In deep kernel learning, we use a neural network to learn a nonlinear feature transformation, and a Gaussian process is defined over the resulting feature space using a traditional kernel function. This enables learning a flexible, data-driven kernel that combines the expressiveness of deep learning with the uncertainty estimation of Gaussian processes \citep{wilson2016deep}.
    However, to the best of the authors' knowledge, learning a Gaussian process with a nonlinear PDE prior is not yet possible \citep{jidling2017linearly}; we propose a solution in this case.
\end{remark}

\begin{remark}[On the warm start]
Pre-training a standard PINN for several hundred epochs provides initial network parameters for basis generation. This is necessary for homogeneous PDEs to prevent convergence to $u \equiv 0$, since standard initialization yields near-zero outputs that trivially minimize the lower-level problem. For non-homogeneous PDEs, the source term prevents this issue.
    
\end{remark}

\subsection{Extension to nonlinear PDEs}

\begin{wrapfigure}{R}{0.45\textwidth}
    \begin{minipage}{0.45\textwidth}
        \vspace*{-0.85cm}
        \begin{algorithm}[H]
            \caption{IFeF-PINN for nonlinear PDEs}
            \label{alg:nonlinear}
            \begin{adjustwidth}{-1em}{}
            \begin{algorithmic}
            \STATE \textbf{Initialize} network parameter $w_0, \theta_0$ and $B$ 
            \FOR{$k$ from $0$ \TO $N_{epoch}$}
                \STATE Formulate extended RFF basis $\psi_D$
                \STATE \textbf{Lower update:}
                \IF{$k$ mod $N_{lower} = 0$}
                \STATE          
                $\theta_{k+1} \approx \arg\min\limits_{\theta}\ \Llower(\omega_k \mid \theta_k  )$
                \ELSE 
                \STATE $ \theta_{k+1} = \theta_k$
                \ENDIF
                \STATE \textbf{Upper update:} 
                \vspace*{-0.15cm}
                \[
                    \omega_{k+1} = \omega_k - \eta_{\omega}\nabla_\omega \Lupper(\wb_k \ | \ \theta_{k+1})
                \]
                \vspace*{-0.4cm}
            \ENDFOR
            \RETURN $\omega_{N_{epoch}}$, $\theta^\star(\omega_{N_{epoch}})$
        \end{algorithmic}
        \end{adjustwidth}
        \end{algorithm}
        \vspace*{-0.8cm}
    \end{minipage}
\end{wrapfigure}

For nonlinear PDEs, the physics residual term $\frac{\lambda}{N_f} \sum_{i=1}^{N_f} \|\mathfrak{F}[u_{\omega,\theta}](x_f^i)\|^2$ becomes nonlinear in $\theta$, making the lower-level problem $\Llower(\theta \ | \ \omega)$ non-convex and lacking a closed-form solution. We therefore replace the exact solution in Proposition~\ref{prop:QP} with gradient descent to find an approximate local minimizer when the Second-Order Sufficient Condition (SOSC) holds, i.e., when the gradient vanishes and the Hessian is positive definite. The complete update is given in Algorithm~\ref{alg:nonlinear}. For computational efficiency, we update $\theta$ to a local minimizer every $N_{\text{lower}}$ epochs. For initialization, we can either warm start only the network parameters $\omega$ via standard PINN pre-training as in the linear case, or initialize both $\omega$ and $\theta$ jointly via end-to-end training as discussed in Section~\ref{sec:e2e}.

\section{Theoretical Analysis}
\label{sec:theoretical_analysis}

\subsection{Convergence properties of the Bi-level algorithm}

We establish convergence by showing that the optimal lower-level solution $\theta^\star(\omega)$ is Lipschitz continuous with respect to the upper-level parameters $\omega$, which ensures a well-defined Lipschitz hypergradient for gradient descent on the upper level.

\begin{proposition}[Lipschitz Continuity of the Solution Map]
\label{thm:lipschitz_theta}
Let the lower-level problem be a strongly convex QP problem parameterized by $\wb$. Assume that the mappings $\wb \mapsto Q(\wb)$ and $\wb \mapsto c(\wb)$ are locally Lipschitz continuous, and that the smallest eigenvalue of $Q(\wb)$ is uniformly bounded below by $\mu_Q > 0$ on any compact set of $\wb$. Then, the optimal solution map $\theta^\star(\wb)$ is also locally Lipschitz continuous with respect to $\wb$. 
\end{proposition}

The detailed proof is provided in Appendix C.2.%
This also holds in the nonlinear PDE cases, when the SOSC is satisfied, the local minimizer $\theta^{\star}(\omega)$ retains Lipschitz continuity and differentiability in a neighborhood of $\omega$. Consequently, the hypergradient is L-smooth, which we leverage in our convergence analysis.

\begin{theorem}[Convergence to a stationary point]
\label{thm:convergence}
Assume that 1) the functions $Q$ and $c$ are continuously differentiable with respect to $\wb$, the upper-level loss $\Lupper$ is continuously differentiable with respect to both $\theta$ and $\wb$; 2) The lower-level problem is $\mu$-strongly convex; 3) the objective function $\Lupper(\cdot \ | \ \theta)$ is bounded below and its hypergradient is L-smooth.

Then, the sequence of iterates $\{\wb_k\}_{k=0}^{\infty}$ generated by the gradient descent algorithm with a constant step size $\eta \in (0, 2/L)$ converges to a stationary point of $\Lupper(\cdot \ | \ \theta)$.

\end{theorem}

The assumptions made are classical in learning problems and are a direct consequence of the structure of the bi-level framework. A formula for the hypergradient is derived via the Implicit Function Theorem in Appendix C.1, %
showing it as a composition of smooth functions. Its Lipschitz continuity is then guaranteed by the Lipschitz continuity 
of the solution map $\theta^\star$ established in Proposition~\ref{thm:lipschitz_theta}.  

\subsection{Universal approximation capabilities}
\label{sec:projection_error}
To analyze the expressiveness of the RFF-augmented features, we show that the hypothesis class is not less expressive than linear readouts over the last hidden layer features. The necessary function spaces for this analysis are defined with comprehensive foundational definitions and proofs in Appendix D.%

\begin{definition}
\label{def:nn_space}
The \textbf{feature space} $\mathcal{H}_{f}$ and the \textbf{composite RFF function space} $\mathcal{H}_{\mathrm{RFF}}$ are defined as:
\begin{equation}
    \label{eq:RFF_space}
    \mathcal{H}_{f} := \left\{ g \ | \ g(x) = h_{\omega}(x)^{\top} \theta, \ \theta \in \mathbb{R}^p \right\}, \quad \mathcal{H}_{\mathrm{RFF}} := \left\{ g \ | \ g(x) =  \psi_D(x)^{\top} \theta, \ \theta \in \mathbb{R}^{2D} \right\},
\end{equation}
where $\psi_D = \gamma_D \circ h_{\omega}$ denotes the vector of composite RFF features defined in Equation~\ref{eq:rff_mapping}.
\end{definition}

We will show that $\overline{\mathcal{H}_{\mathrm{RFF}}}$ strictly contains $\mathcal{H}_{f}$, and thus defines a more expressive hypothesis class. The argument constructs a bridge between the two spaces using a reproducing kernel Hilbert space. 

\begin{theorem}
\label{thm:projection_error}
Let $f$ be any target function in $\mathcal{L}^2(\Omega,\R)$. The projection error (see Definition 3 %
in D.1) %
achievable by the composite RFF Function Space $\mathcal{H}_{\mathrm{RFF}}$ is no greater than the projection error achieved by the original Feature Space $\mathcal{H}_{f}$ when the number of RFF features $D$ goes to infinity. %

\end{theorem}

The proof is given in Appendix D.2. %
This result establishes a powerful theoretical assurance that RFF embedding offers better approximation capabilities. 
Theorem~\ref{thm:projection_error} yields the universal approximation corollary presented below, the proof of which is given in Appendix D.2.1%

\begin{corollary}[Universal approximation]
    The projection error of the solution $u$ to \eqref{equ:general_pde} onto $\mathcal{H}_{\mathrm{RFF}}$ can be made as small as desired, provided enough neurons and RFF features $D$.
\end{corollary}

\section{Related Work}
\label{sec:sota}

\paragraph{Weight-balancing strategies} These methods adapt the physics weight $\lambda$ in \eqref{eq:pinn_sampled_loss} during training. For instance, \citep{wang2021understanding} dynamically updates $\lambda$ to balance the gradients of data and physics losses, while the NTK framework \citep{jacot2018neural,krishnapriyan2021characterizing} enforces equal decay rates, theoretically recovering high-frequency solutions. Primal–dual methods \citep{goemans1997primal,barreau2025accuracy} instead compute $\lambda$ from the PDE residual. Although simple to implement, these approaches offer weak convergence guarantees and remain tied to single-level optimization. Nonetheless, they are complementary to our framework and could be integrated as weight-balancing strategies within the upper-level problem.

\paragraph{Resampling strategies} A second line of work reduces the gap between the true loss $\mathfrak{L}_{\lambda}$ and its sampled counterpart $\hat{\mathfrak{L}}_{\lambda}$. Examples include NTK-informed sampling \citep{lau2024pinnacle}, adversarial sampling \citep{tang2024adversarial}, and reinforcement learning \citep{song2025rl}. While effective in reducing approximation error, these methods do not explicitly target spectral bias, which is the focus of our proposed method.

\paragraph{Curriculum learning strategies} Finally, new architectures and training schedules aim to better capture high-frequency components. Attention mechanisms \citep{sirignano2018dgm}, multi-stage networks \citep{AmandaStacked,waheed2022kronecker,chai2024overcoming,mustajab2024physics,eshkofti2025vanishing,wang2024multi}, or finite-basis approximation \citep{moseley2023finite} have shown improved multi-scale resolution. However, their complexity often makes training slow and delicate, and they still lack dedicated optimization algorithms.

\section{Numerical Experiments}
\label{sec:experiments}

\paragraph{Objective.}
In this section, we describe comprehensive experiments that establish four main advantages of IFeF-PINN. First, improved approximation over PINNs and SOTA variants on low-frequency PDEs. Second, higher accuracy on high-frequency and multi-scale linear PDEs, where standard PINNs typically show failure modes. Third, our framework exhibits strong generalization capabilities when integrated with advanced PINN variants. Finally, a spectrum analysis experiment demonstrates that our proposed method improves the network fitting accuracy for high-frequency signals. %

\paragraph{Experiment setup.}
We will use four PDEs, namely the 2D Helmholtz equation (low and high frequency), 1D convection equation (low and high frequency), 1D convection-diffusion equation, and the viscous Burgers' equation.
The baseline methods are Vanilla PINNs, NTK~\citep{wang2022and}, PINNsformer~\citep{ZhaoEtAl24}, and Physics-Informed Gaussians (PIG)~\citep{KangEtAl25}, keeping their default settings for a fair comparison. Additional experimental comparisons with Multiple Fourier Features (MFF)~\citep{wang2021eigenvector} are provided in Appendix G.1. %
For simplicity, we set $\lambda = 0.01$ for the Vanilla PINNs in Equation~\ref{eq:pinn_sampled_loss}. Detailed hyperparameters for our proposed methods are in Appendix E. %
For low-frequency 2D Helmholtz and low-frequency 1D convection equations, we adopt the uniform sampling strategy settings of~\cite{ZhaoEtAl24}. For the viscous Burgers' equation, we follow the setup of~\cite{raissi2019physics}. For the high-frequency Helmholtz equation, we employ Latin hypercube sampling \citep{mckay2000comparison} to improve domain coverage. 
We evaluate two variants of our framework: IFeF (Vanilla training) and IFeF-PD (primal-dual weight-balancing proposed by \cite{barreau2025accuracy}). PDE definitions, datasets, and network architectures are provided in Appendix F. %
We measure the relative $L^2$-error after convergence, defined as $\frac{\lVert u_\text{pred} - u_\text{real} \rVert_2}{\lVert u_\text{real} \rVert_2}$. Each method is run five times with independent random seeds, with the best predictions for each approach. All models are implemented in PyTorch and trained on a single NVIDIA GeForce RTX 4090 GPU. The code for all benchmarks is available at \url{https://github.com/CyberAltrumi/IFeF-PINN}. Computational aspects are evaluated in Appendix G.2.

\subsection{Results on benchmark PDEs}
We begin with three popular low-frequency benchmark PDEs: 2D Helmholtz equation, 1D convection equation, and the viscous Burgers' equation. Figure~\ref{fig:lowfi_boxplot} summarizes relative $L^2$-errors across baseline methods; box plots display medians and IQRs, and red diamonds denote means. Additional prediction and absolute error maps are provided in Appendix G. %
\begin{figure}[htbp]
    \centering
        \includegraphics[width=0.99\linewidth]{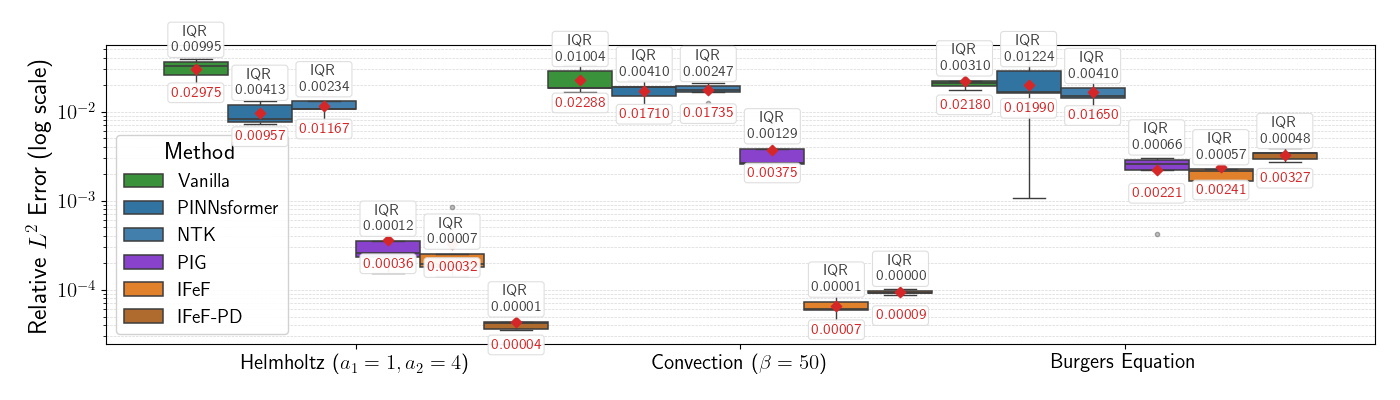}
    \caption{Box plot of relative $L^2$-errors (log10 scale) for all methods on three low-frequency benchmarks with median, inter-quartile range (IQR), and mean (red diamonds).}
    \label{fig:lowfi_boxplot}
    \vspace*{-0.5cm}
\end{figure}

Across all these problems, our proposed method attains the lowest median errors with reduced variability. On Helmholtz, IFeF-PD achieves the best relative $L^2$ error of $3.5 \times10^{-5}$. On convection, IFeF achieves the best error of $4.3 \times10^{-5}$. Even in the nonlinear case of Burgers' equation, IFeF obtains the lowest median error.
In addition, we conducted an ablation study where we discarded the RFF basis extension but performed a similar iterative two-step optimization process, obtaining results that were similar but slightly better than those of the Vanilla PINN ($1.4923\times10^{-2}$ relative $L^2$-error) on the low-frequency convection problem.
Figure~\ref{fig:Helmholtz_lowfi} presents the predictions for the low-frequency 2D Helmholtz case. On a logarithmic scale, the gap between IFeF-PINN and other methods is consistent with the box plot summaries.
These results highlight the strong approximation capability of the proposed method, especially for linear equations, underscoring its robustness for solving diverse PDEs.

\begin{figure}[htbp]
    \centering
        \includegraphics[width=0.99\linewidth]{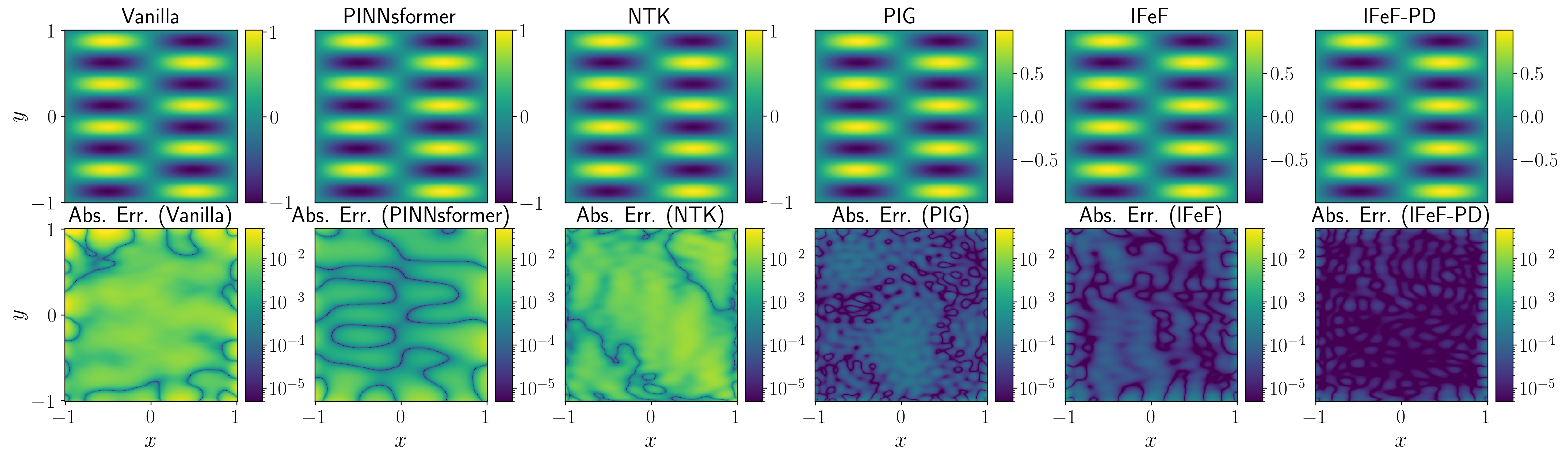}
    \caption{Low-frequency Helmholtz equation prediction solution (up) and absolute error on a log10 scale (bottom) of baseline methods.}
    \label{fig:Helmholtz_lowfi}
    \vspace*{-0.5cm}
\end{figure}

\subsection{Mitigating the spectral bias}
To evaluate challenging cases of spectral bias, we study the failure modes of PINNs on high-frequency and multi-scale PDEs, where vanilla PINNs typically struggle to learn rapidly oscillatory or widely separated frequency components.
In particular, we study the high-frequency Helmholtz and convection equations, as well as a multi-scale convection-diffusion equation. Table~\ref{tab:err_hifi} presents the mean and standard deviation of the relative $L^2$-errors over baselines applied to these problems. Additional prediction and absolute error maps are provided in Appendix G.%

\begin{table}[htbp]
  \centering
  \setlength{\tabcolsep}{10pt}
  \renewcommand{\arraystretch}{1.15}
  \begin{tabular}{l c c c }
    \toprule
    Baseline & \shortstack{Helmholtz \\ ($a_1=a_2=100$)} & \shortstack{Convection \\ ($\beta=200$)} & \shortstack{Convection-Diffusion \\ ($k_\text{low} = 4\pi$, $k_\text{high} = 60\pi$)} \\
    \midrule
    Vanilla     & - & 0.9024 (0.0239) & 0.0501 (0.0030) \\
    PINNsformer  & - & 1.2278 (0.2010)  & 0.0525 (0.0001) \\
    NTK        & - & 0.8685 (0.0318) &  0.0526 (0.0001) \\
    PIG   & 1.6884 (0.2775) & 1.0009 (0.0003) & 0.0560 (0.0010)\\
    IFeF       & 0.0156 (0.0055)  & 0.0027 (0.0010) & \textbf{0.0009} (0.0003)\\
    IFeF-PD       & \textbf{0.0092} (0.0031) & \textbf{0.0025} (0.0005)  & 0.0010 (0.0002)\\
    \bottomrule
  \end{tabular}
  \caption{Average relative $L^2$-error with corresponding standard deviation for each baseline on three high-frequency PDEs. A dash '-' denotes that the baseline failed to converge.}
  \label{tab:err_hifi}
\end{table}

\begin{figure}[htbp]
    \centering
        \includegraphics[width=0.99\linewidth]{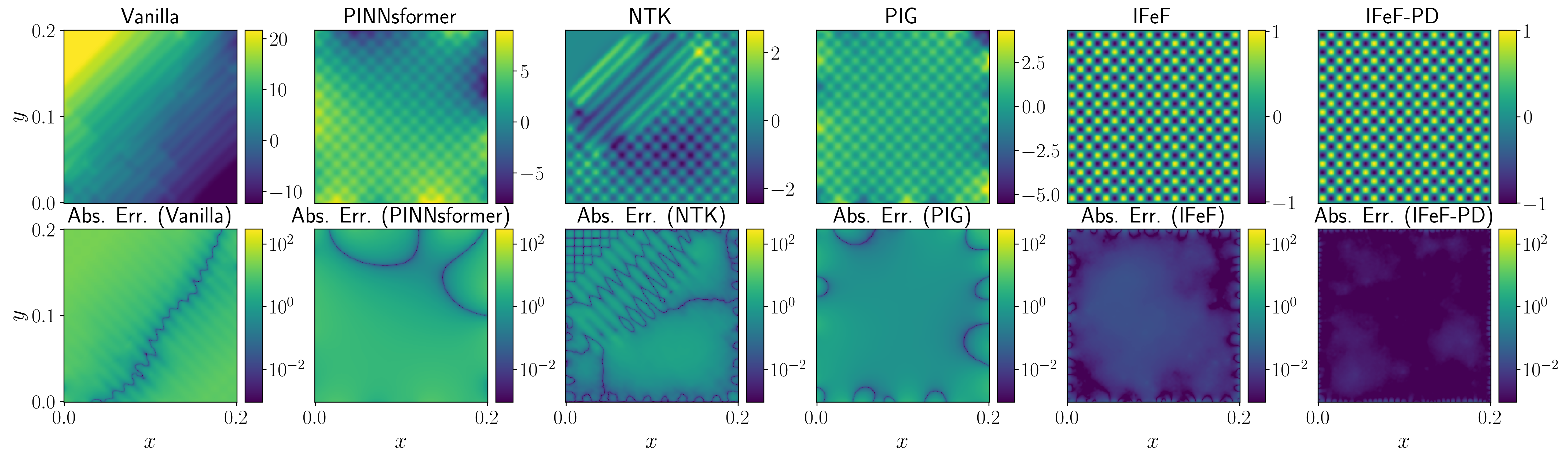}
    \caption{High-frequency Helmholtz equation prediction solution (up) and absolute error in log scale (bottom) of baseline methods.}
    \label{fig:Helmholtz_hifi}
    \vspace*{-0.5cm}
\end{figure}

Figure~\ref{fig:Helmholtz_hifi} depicts the high-frequency Helmholtz solutions and the corresponding log-scale absolute errors. 
In the considered scenarios, all baselines exhibit clear failure modes.
We also conducted a similar ablation study as described in the previous section, removing the RFF basis extension, and the training did not converge for both the high-frequency Helmholtz and convection equations.
In contrast, the proposed IFeF-PINN method effectively mitigates the spectral bias of neural networks.
Moreover, when combined with the primal-dual method to adaptively balance the physics-based loss, our method achieves accurate solutions even under very high frequencies, which illustrates the flexibility of the proposed framework in incorporating advanced learning methods.
A similar result holds for the multi-scale convection-diffusion equation in Figure 4 %
in Appendix G, %
clearly showing that only IFeF-PINN succeeds in learning both low and high frequency components of the solution. In contrast, all baselines suffer from the spectral bias failure mode, where models prioritize learning low-frequency components and tend to ignore the high-frequency components. 

\subsection{Spectrum Analysis}
To quantitatively demonstrate our method's ability to mitigate spectral bias, we employ the fast Fourier transform to analyze the frequency-domain distribution of the network's prediction.
We conduct a spectrum analysis similar to~\cite{rahaman2019spectral}, designing a challenging multi-scale convection equation with an initial condition composed of a superposition of ten sinusoids of different frequencies and unit amplitude. More details of the setup are in Appendix F.2.%
\begin{wrapfigure}{r}{0.45\textwidth}
    \centering
    \includegraphics[width=1\linewidth]{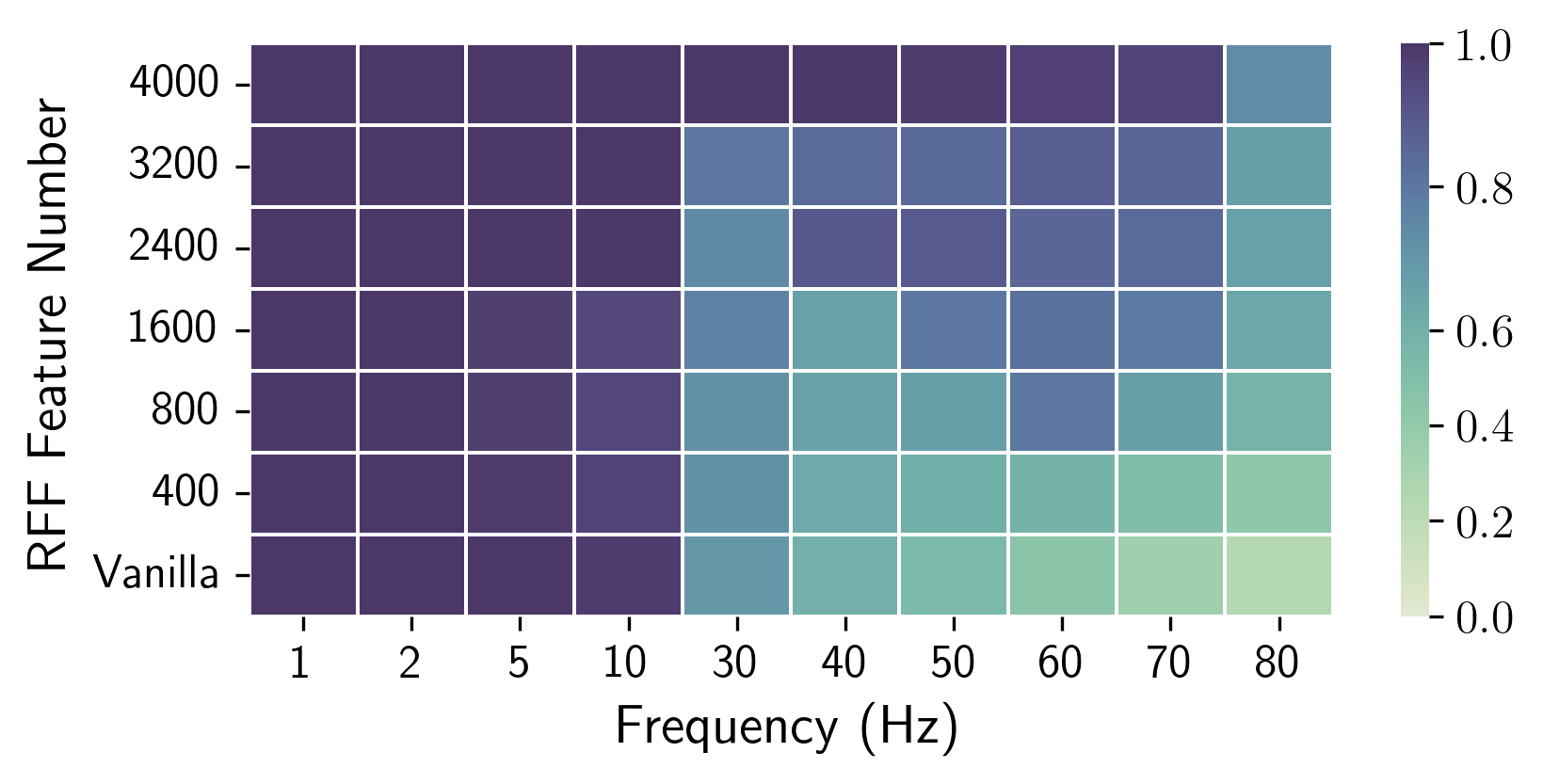}
    \caption{Prediction of the network spectrum with an increasing number of Fourier features. The x-axis represents frequency, and the colorbar shows the normalized magnitude of the predicted solution at $t=0$. The colorbar is scaled accordingly from 0 to 1.}
    \label{fig:spectrum}
    \vspace*{-0.6cm}
\end{wrapfigure}

During analysis, we compare the performance of Vanilla PINNs against models where the basis is extended with a varying number of random Fourier features and carry out a one-step solution of the lower-level objective in Equation~\ref{equ:Qcb_loss}. No additional training is performed for the upper-level problem.

We compute the magnitude of their discrete Fourier transform at frequencies $k_i$, denoted as $|\tilde{f}_{k_i}|$.
Figure~\ref{fig:spectrum} presents the average normalized magnitudes $\frac{|\tilde{f}{k_i}|}{A_i}$ over five independent runs. The results clearly illustrate the spectral bias of Vanilla PINNs, which struggle to accurately capture high-frequency components. In contrast, by extending the network's basis through RFF, the network can fit high-frequency signals much more effectively, even without the subsequent bi-level training procedure of IFeF-PINN. Furthermore, we observe that increasing the number of random features enhances the network’s ability to approximate high-frequency components, confirming the effectiveness of our basis extension strategy.

\subsection{Ablation studies}
In this section, we present experiments to demonstrate the effects of two-stage training in IFeF-PINN and the number of Fourier-enhanced features and the Gaussian sampling parameter $\sigma$.

\subsubsection{End-to-end training}
\label{sec:e2e}
To validate the necessity of two-stage training in IFeF-PINN, we conduct an end-to-end ablation where both network parameters $\omega$ and coefficients $\theta$ are jointly optimized. We keep the approximation in Equation~\ref{equ:u_basis_approx} but incorporate $\theta$ as learnable parameters alongside $\omega$, and directly minimize $\hat{\mathfrak{L}}_{\lambda}(u_{\omega, \theta})$ without the two-stage training. Unlike IFeF-PINN where $\theta$ is always optimal under current features $\psi_D(x)$, $\theta$ is randomly initialized and updated simultaneously with $\omega$, losing the optimality guarantee. Table~\ref{tab:e2e_training} presents the results for low- and high-frequency Helmholtz and Burgers' equations. This ablation validates the necessity of two-stage training, as IFeF-PINN significantly outperforms end-to-end training in linear PDEs through guaranteed lower-level optimality of $\theta$, while showing modest improvements in nonlinear PDEs where the lower-level becomes non-convex.

\begin{table}[htbp]
  \centering
  \setlength{\tabcolsep}{10pt}
  \renewcommand{\arraystretch}{1.15}
  \begin{tabular}{l c c c }
    \toprule
    Ablation & \shortstack{Helmholtz \\ ($a_1=1, a_2=4$)}  & \shortstack{Helmholtz \\ ($a_1=a_2=100$)}  & \shortstack{Viscous Burgers \\ ($\nu = \frac{0.01}{\pi} $)} \\
    \midrule
     End-to-End      &$0.0088 (0.0006)$ & -  & $0.0049 (0.0009)$  \\
     IFeF      &$0.0003(0.0003) $ & $0.0156(0.0055) $   & $0.0024(0.0011) $  \\
     IFeF-PD      &$0.00005(0.00002) $ & $0.0092(0.0031) $   & $0.0033(0.0004) $  \\
    \bottomrule
  \end{tabular}
  \caption{Average relative $L^2$-error with corresponding standard deviation for end-to-end training and IFeF-PINN on three benchmarks. A dash '-' denotes that the baseline failed to converge.}
  \label{tab:e2e_training}
  \vspace*{-0.5cm}
\end{table}

\subsubsection{Hyperparameter ablation}
We conduct an ablation on two key hyperparameters in IFeF-PINN: the number of Fourier features $D$ and the Gaussian sampling parameter $\sigma$. We evaluate their impact on performance using the low- and high-frequency Helmholtz equations, with results shown in Table~\ref{tab:hyperparameter_ablation}. The ablation shows that too few features reduce expressivity while excessive features cause overfitting and may break the rank condition discussed in Appendix B.1. For $\sigma$, larger values are essential for high-frequency problems discussed in~\cite{tancik2020fourier, wang2021eigenvector}. Low-frequency problems are robust to both hyperparameters, while high-frequency problems are sensitive, especially to $\sigma$.

\begin{table}[htbp]
  \centering
  \small
  \setlength{\tabcolsep}{8pt}
  \renewcommand{\arraystretch}{1.1}
  \begin{tabular}{l c c c c c}
    \toprule
    \multicolumn{6}{c}{Helmholtz ($a_1=1, a_2=4$)} \\
    \midrule
    $D$ ($\sigma=1$) & 100 & 400 & 800 & 1200 & 3000  \\
    Rel. $L^2$ error & $5.5\times 10^{-4}$ & $\mathbf{2.1\times 10^{-4}}$ & $3.2\times 10^{-4}$ & $5.7\times 10^{-4}$ & $4.5\times 10^{-4}$ \\
    \midrule
    $\sigma$ ($D=800$) & 2 & 1 & 0.5 & 0.2 & 0.1  \\
    Rel. $L^2$ error & $4.0\times 10^{-4}$ & $\mathbf{3.2\times 10^{-4}}$ & $5.5\times 10^{-4}$ & $3.3\times 10^{-4}$ & $1.5\times 10^{-3}$ \\
    \midrule
    \multicolumn{6}{c}{Helmholtz ($a_1=a_2=100$)} \\
    \midrule
    $D$ ($\sigma=1$) & 800 & 1200 & 1600 & 2400 & 3000  \\
    Rel. $L^2$ error & $7.11\times 10^{-2}$ & $5.40\times 10^{-2}$ & $3.09\times 10^{-2}$ & $\mathbf{1.56\times 10^{-2}}$ & $2.22\times 10^{-2}$  \\
    \midrule
    $\sigma$ ($D=2400$) & 20 & 10 & 5 & 1 & 0.2 \\
    Rel. $L^2$ error & $4.6\times 10^{-3}$ & $\mathbf{3.0\times 10^{-3}}$ & $5.7\times 10^{-3}$ & $1.56\times 10^{-2}$ & $1.05\times 10^{-1}$ \\
    \bottomrule
  \end{tabular}
  \caption{Average relative $L^2$-error for hyperparameter ablation for $D$ and $\sigma$ on Helmholtz equations.}
  \label{tab:hyperparameter_ablation}
  \vspace*{-0.5cm}
\end{table}

\section{Conclusion}
\label{sec:conclusion}
In this paper, we introduce IFeF-PINN, a novel iterative training method for Fourier-enhanced Features PINNs. By augmenting the network with random Fourier features mapping as a basis extension with the bi-level problem, IFeF-PINN mitigates the spectral bias problem of standard PINNs when capturing the high-frequency and multi-scale components during training. Experimental results demonstrate that IFeF-PINN consistently outperforms advanced baselines across various scenarios, including popular low-frequency benchmarks and handling high-frequency and multi-scale PDEs. Furthermore, it has strong flexibility when integrating with different training strategies for PINNs.

Despite its strengths, IFeF-PINN faces challenges when extended to nonlinear PDEs. For nonlinear PDEs, the lower-level problem becomes nonconvex, precluding a one-step solve and requiring iterative two-stage gradient descent updates that can stall in local minima. Advancing principled bi-level optimization techniques to better handle the nonlinear lower-level problem remains a promising direction for future work.

\section{Acknowledgments}
This work was partially supported by the Wallenberg AI, Autonomous Systems and Software Program (WASP), funded by the Knut and Alice Wallenberg Foundation. It was further supported by the Swedish Research Council through the Distinguished Professor Grant 2017-01078, as well as by the Wallenberg Scholar Grant from the Knut and Alice Wallenberg Foundation. The authors also gratefully acknowledge the support of Digital Futures.

\bibliography{iclr2026_conference}
\bibliographystyle{iclr2026_conference}

\clearpage
\appendix

\section{Contents}
\label{app:app_contents}
We organize this supplementary document as follows:
\begin{itemize}
    \item Section~\ref{app:qp-derivation} provides the formulation for the QP problem and proof of Proposition 1.
    \item Section~\ref{app:convergence_analysis} provides the proofs for the convergence analysis of our method.
    \item Section~\ref{app:proofs_projection} provides the proofs for the projection error analysis of our method.
    \item Section~\ref{app:hyper_setting} details the hyperparameter settings of our method.
    \item Section~\ref{app:pdes_setup} describes the experimental setup for all PDEs, models, and datasets.
    \item Section~\ref{app:additional_results} presents additional results and computational costs considered in our numerical experiments.
\end{itemize}

\section{QP formulation details and proof of Proposition 1}%
\label{app:qp-derivation}

Let $X_u = (x_u^1, \dots x_u^{N_u}) \in \mathbb{R}^{N_u \times n}$ be boundary collocation points and $X_f = (x_f^1, \dots, x_f^{N_f}) \in \mathbb{R}^{N_f \times n}$ be interior collocation points. We define
\begin{itemize}
    \item $B_u(\omega) \in \R^{N_u \times 2D}$ with rows $\mathfrak{B}[\psi \circ h_\omega](x_u^i)^{\top} \in \mathbb{R}^{n}$ for $i = 1, \dots, N_u$ as boundary values;
    \item $G_u \in \R^{N_u}$ with rows $g(x_u^i)^{\top} \in \mathbb{R}^{n}$ for $i = 1, \dots, N_u$ as boundary measurement;
    \item $R_f(\omega) \in \R^{N_f \times 2D}$ with rows $\mathfrak{F}[\psi \circ h_{\omega}](x_f^i)^{\top}$ for $i = 1, \dots, N_f$ as residual values;
    \item $F_f \in \R^{N_f}$ with rows $f(x_f^i)^{\top}$ for $i = 1, \dots, N_f$ as source terms.
\end{itemize}

Note that the linearity of the $\mathfrak{B}$ and $\mathfrak{F}$ operators leads to $\mathfrak{B}[u_{\omega,\theta}](X_u) = B_u \theta$ and $\mathfrak{F}[u_{\omega,\theta}](X_f) = R_f \theta$. Consequently, using the standard PIML loss with boundary and PDE residual terms, we get:
\begin{equation}
    \hat{\mathfrak{L}}_{\lambda}(u_{\omega,\theta})
    = \frac{1}{N_u} \bigl\| B_u(\omega) \theta - G_u \bigr\|^2
    + \frac{\lambda}{N_f}\bigl\| R_f(\omega) \theta - F_f \bigr\|^2.
\end{equation}
The loss expands to the quadratic form
\begin{multline}
    \label{equ:qplossfun-appendix}
    \Llower(\theta \ | \ \omega) = \frac{1}{2}\theta^\top \Bigl(\frac{2}{N_u}B_u^\top B_u + \frac{2\lambda_{\mathrm{LL}}}{N_f}R_f^\top R_f\Bigr)\theta + \Bigl(-\frac{2}{N_u}B_u^\top G_u - \frac{2\lambda_{\mathrm{LL}}}{N_f}R_f^\top F_f\Bigr)^\top \theta \\
    + \frac{1}{N_u} G_u^\top G_u + \frac{\lambda_{\mathrm{LL}}}{N_f} F_f^\top F_f.
\end{multline}
Identifying
\begin{equation*}
    Q(\omega) = \frac{2}{N_u} B_u(\omega)^\top B_u(\omega) + \frac{2\lambda_{\mathrm{LL}}}{N_f} R_f(\omega)^\top R_f(\omega), \quad 
    c(\omega) = -\frac{2}{N_u} B_u(\omega)^\top G_u - \frac{2\lambda_{\mathrm{LL}}}{N_f} R_f(\omega)^\top F_f,
\end{equation*}
leads to $\Llower(\theta \ | \ \omega) = \tfrac{1}{2}\theta^\top Q(\omega) \theta + c(\omega)^\top \theta + b,$ where $b$ is a constant term and $\lambda_{\mathrm{LL}}$ is the physics weight used in the lower-level problem.

\subsection{Analysis of the Rank Condition and Regularization of Proposition 1}%
\label{app:proof_rank_cond}

The positive semi-definiteness of $Q(\omega)$ follows from the factorization
$Q(\omega) = M(\omega)^\top M(\omega) \in \mathbb{R}^{2D\times 2D}$, with the stacked design matrices $M(\omega)$ as follows:
\[ M(\omega) =  \begin{pmatrix} \sqrt{\frac{2}{N_u}}B_u(\omega) \\   \sqrt{\frac{2\lambda}{N_f}}R_f(\omega) \end{pmatrix},\]
where $\lambda > 0$ is always ensured. Then we have $\operatorname{rank}(Q(\omega)) = \operatorname{rank}(M(\omega))$. Hence, $Q(\omega)$ is strictly positive definite if and only if  $M(\omega)$ has full column rank, i.e.,
\begin{equation}
    \label{eq:rank_condition}
    \operatorname{rank}(M(\omega)) = 2D.
\end{equation}

The loss $\Llower(\theta \ | \ \omega)$ becomes then strongly convex, and its minimization has a unique solution $\theta^\star = -Q^{-1}c$.

The rank condition in \eqref{eq:rank_condition} imposes a fundamental constraint: for the stacked design matrix $M \in \mathbb{R}^{(N_u + N_f) \times 2D}$ to have full column rank $2D$, it is necessary that the number of rows is at least as large as the number of columns. This leads to the critical requirement on the number of sampling points: $N_u + N_f \geq 2D$.

When this condition is violated (i.e., $N_u + N_f < 2D$), the system becomes underdetermined. This directly causes the matrix $Q({\omega})$ to be rank-deficient. 
Consequently, the lower-level problem becomes ill-posed, lacking a unique solution as $Q^{-1}$ does not exist. This might lead to an aliased solution (especially true when we add a Tikhonov regularization).

A more subtle cause of rank deficiency occurs even when $N_u + N_f \geq 2D$. This happens if the collocation points provide redundant information, failing to create $2D$ linearly independent constraints. Such a situation can arise from geometrically poor sampling (e.g., points lying on nodal lines of the basis functions) or from inherent redundancies in the randomly generated RFF basis itself (e.g., two different random vectors being nearly parallel). In these scenarios, although the design matrix $M$ has enough rows, its columns remain linearly dependent, leading to a singular or, more commonly, a numerically ill-conditioned matrix $Q$. However, if $D$ is large, a solution might be to just discard these redundancies while still keeping the same feature space $\mathcal{H}_{RFF}$ (see equation 8 of the main paper).%

To resolve this, we employ Tikhonov regularization, modifying the matrix to $Q_{\mathrm{reg}}({\omega}) = Q({\omega}) + \gamma {I}$, where $\gamma > 0$ is a small regularization parameter. This ensures $Q_{\mathrm{reg}}$ is always positive definite and invertible, since for any non-zero $\theta$, the quadratic form $\theta^\top Q_{\mathrm{reg}}\theta = \theta^\top Q\theta + \gamma\|\theta\|^2$ is strictly positive. The lower-level problem thus regains a unique, stable solution
\begin{equation}
\label{eq:regularized_solution}
\theta^{\star}({\omega}) = -(Q({\omega}) + \gamma {I})^{-1} {c}({\omega}).
\end{equation}

This analysis reveals a fundamental trade-off: while increasing $D$ enhances representational power, it demands proportionally more sampling points to maintain a well-posed system. When the sampling budget is limited, Tikhonov regularization provides a principled remedy to ensure algorithmic stability, at the cost of introducing a slight bias to the solution.

\section{Proofs of Convergence Analysis}
\label{app:convergence_analysis}
\subsection{Hypergradient Derivation via Implicit Function Theorem (IFT)}
\label{app:hypergradient}

The hypergradient $\nabla_{\wb}\Lupper(\wb)$ is computed using the chain rule, where the Implicit Function Theorem (IFT) provides the Jacobian of the lower-level solution map $\theta^\star(\wb)$.

The lower-level optimality condition is $F(\theta^\star, \wb) := Q(\wb)\theta^\star + c(\wb) = 0$. Taking the total derivative with respect to $\wb$ yields
\begin{equation*}
    \frac{\partial {F}}{\partial \theta^\top} \frac{\partial \theta^\star}{\partial \wb^\top} + \frac{\partial {F}}{\partial \wb^\top} = {0}.
\end{equation*}
Solving for the Jacobian $\frac{\partial \theta^\star}{\partial \wb^\top}$ gives
\begin{equation*}
    \frac{\partial \theta^\star}{\partial \wb^\top} = - \left( \frac{\partial {F}}{\partial \theta^\top} \right)^{-1} \frac{\partial {F}}{\partial \wb^\top} = -{Q}(\wb)^{-1} \left( \frac{\partial({Q}(\wb)\theta^\star)}{\partial \wb^\top} + \frac{\partial{c}(\wb)}{\partial \wb^\top} \right).
\end{equation*}
The full hypergradient is then obtained by the chain rule
\begin{equation}
    \nabla_{\wb}\Lupper(\wb) = \frac{\partial \Lupper}{\partial \wb} + \left( \frac{\partial \theta^\star}{\partial \wb^\top} \right)^\top \frac{\partial \Lupper}{\partial \theta}.
\end{equation}
Substituting the expression for the Jacobian, we get
\begin{equation*} \label{eq:hypergradient}
    \gradw \Lupper(\wb) = \frac{\partial \Lupper} {\partial \wb} - \left( \frac{\partial ({Q}(\wb)\theta^*)}{\partial \wb^\top} + \frac{\partial {c}(\wb)}{\partial \wb^\top} \right)^\top {Q}(\wb)^{-1} \frac{\partial \Lupper}{\partial \theta}.
\end{equation*}
This provides a computable formula for the gradient used in the upper-level optimization.

\subsection{Proof of Proposition 2}%
\label{app:proof_lipschitz_theta}

We first state the key properties required for our convergence analysis.

\begin{assumption}[Smoothness and Boundedness Properties] 
\label{as:smooth}
The bi-level optimization problem satisfies the following regularity conditions:
\begin{enumerate}
    \item The functions $Q(\omega)$ and $c(\omega)$ are continuously differentiable with respect to $\omega$. The upper-level loss $\Lupper(\theta, \omega)$ is continuously differentiable with respect to both $\theta$ and $\omega$.
    \item The lower-level problem is $\mu$-strongly convex, i.e., $Q(\omega) \succeq \mu I$ for some constant $\mu > 0$.
    \item The objective function $\Lupper(\omega)$ is bounded below by a scalar $\mathfrak{L}_{\inf}$.
\end{enumerate}
\end{assumption}

\begin{assumption}[L-Smoothness of the Hypergradient] \label{as:lipschitz}
    The upper-level objective function $\Lupper(\omega)$ is L-smooth, a standard assumption in gradient-based optimization analysis. This means its gradient, the hypergradient $\nabla_{\omega} \Lupper(\omega)$, is Lipschitz continuous with constant $L > 0$:
    \begin{equation}
        \|\nabla_{\omega} \Lupper(\omega_1) - \nabla_{\omega} \Lupper(\omega_2)\| \leq L \|\omega_1 - \omega_2\|, \quad \forall \omega_1, \omega_2 \in \mathbb{R}^P.
    \end{equation}
\end{assumption}
These assumptions trivially hold when neural network activation functions are Lipschitz continuous and the loss function is smooth, which is satisfied by our choice of $\tanh$ activations with MSE losses.

\begin{proof}
Since $\Llower$ is strongly convex and differentiable with respect to $\theta$, the unique optimal solution $\theta^\star(\wb)$ is found by setting the gradient to zero:
\begin{equation}
\nabla_{\theta} \Llower(\theta^\star(\wb)) = {Q}(\wb)\theta^\star(\wb) + {c}(\wb) = 0.
\end{equation}
This gives the closed-form solution showed in Proposition 1:%
\begin{equation}
\theta^\star(\wb) = -{Q}(\wb)^{-1}{c}(\wb).
\label{eq:analytic_solution}
\end{equation}
Then consider two parameter vectors $\wb_1, \wb_2$ from a compact set $\mathcal{W}$. We want to bound the norm of the difference $\|\theta^\star(\wb_1) - \theta^\star(\wb_2)\|$. This follows the structure from your provided image:
\begin{align}
    \|\theta^\star(\wb_1) - \theta^\star(\wb_2)\| &= \|{Q}(\wb_1)^{-1}{c}(\wb_1) - {Q}(\wb_2)^{-1}{c}(\wb_2)\| \nonumber \\
    &= \|{Q}(\wb_1)^{-1}{c}(\wb_1) - {Q}(\wb_1)^{-1}{c}(\wb_2) + {Q}(\wb_1)^{-1}{c}(\wb_2) - {Q}(\wb_2)^{-1}{c}(\wb_2)\| \nonumber \\
    &\le \|{Q}(\wb_1)^{-1}({c}(\wb_1) - {c}(\wb_2))\| + \|({Q}(\wb_1)^{-1} - {Q}(\wb_2)^{-1}){c}(\wb_2)\| \nonumber \\
    &\le \|{Q}(\wb_1)^{-1}\| \cdot \|{c}(\wb_1) - {c}(\wb_2)\| + \|{Q}(\wb_1)^{-1} - {Q}(\wb_2)^{-1}\| \cdot \|{c}(\wb_2)\|.
    \label{eq:proof_step1}
\end{align}
We use the matrix identity $A^{-1} - B^{-1} = A^{-1}(B-A)B^{-1}$ to bound the second term:
\begin{align}
    \|{Q}(\wb_1)^{-1} - {Q}(\wb_2)^{-1}\| &= \|{Q}(\wb_1)^{-1}({Q}(\wb_2) - {Q}(\wb_1)){Q}(\wb_2)^{-1}\| \nonumber \\
    &\le \|{Q}(\wb_1)^{-1}\| \cdot \|{Q}(\wb_2) - {Q}(\wb_1)\| \cdot \|{Q}(\wb_2)^{-1}\|.
    \label{eq:proof_step2}
\end{align}
Substituting Equation~\ref{eq:proof_step2} back into Equation~\ref{eq:proof_step1}:
 \begin{align}
\label{equ:theta_lips}
    \|\theta^\star(\wb_1) - \theta^\star(\wb_2)\| &\le \|{Q}(\wb_1)^{-1}\| \cdot \|{c}(\wb_1) - {c}(\wb_2)\| \nonumber \\
    &\quad + \|{Q}(\wb_1)^{-1}\| \cdot \|{Q}(\wb_2) - {Q}(\wb_1)\| \cdot \|{Q}(\wb_2)^{-1}\| \cdot \|{c}(\wb_2)\|.
\end{align}

By Assumption~\ref{as:smooth}, on the compact set $\mathcal{W}$, there exist constants $L_Q, L_c > 0$ such that $\norm{{Q}(\wb_1) - {Q}(\wb_2)} \le L_Q \norm{\wb_1 - \wb_2}$ and $\norm{{c}(\wb_1) - {c}(\wb_2)} \le L_c \norm{\wb_1 - \wb_2}$. Furthermore, due to strong convexity, there is a $\mu_Q > 0$ such that $\norm{{Q}(\wb)^{-1}} \le 1/\mu_Q$ for all $\wb \in \mathcal{W}$. Finally, since ${c}(\wb)$ is continuous on a compact set, its norm is bounded by a constant $C_{\max} = \sup_{\wb \in \mathcal{W}} \norm{{c}(\wb)}$.

Substituting these bounds into Equation~\ref{equ:theta_lips}:
\begin{align*}
    \norm{\theta^\star(\wb_1) - \theta^\star(\wb_2)} &\le \frac{1}{\mu_Q} (L_c \norm{\wb_1 - \wb_2}) + \left( \frac{1}{\mu_Q} \cdot L_Q \norm{\wb_1 - \wb_2} \cdot \frac{1}{\mu_Q} \right) C_{\max} \\
    &= \left( \frac{L_c}{\mu_Q} + \frac{L_Q C_{\max}}{\mu_Q^2} \right) \norm{\wb_1 - \wb_2}.
\end{align*}
Defining the constant $K = \frac{L_c}{\mu_Q} + \frac{   L_Q C_{\max}}{\mu_Q^2}$ completes the proof.
\end{proof}

\subsection{Proof of Theorem 1}%
\label{app:proof_convergence}

The proof relies on the following standard lemma for L-smooth functions.

\begin{lemma}[Sufficient Decrease]
\label{lem:sufficient_decrease}
If $\Lupper(\wb)$ is L-smooth with constant $L$ and the step size $\eta \in (0, 2/L)$, the gradient descent update rule ensures a sufficient decrease in the objective function:
\begin{equation*}
    \Lupper(\wb_{k+1}) \le \Lupper(\wb_k) - \eta \left(1 - \frac{L\eta}{2}\right) \norm{\nabla_{\wb} \Lupper(\wb_k)}^2.
\end{equation*}
\end{lemma}
\begin{proof}
From the L-smoothness property (descent lemma) under Assumption~\ref{as:lipschitz}, we have:
\begin{align*}
    \Lupper(\wb_{k+1}) &\le \Lupper(\wb_k) + \nabla_{\wb} \Lupper(\wb_k)^\top(\wb_{k+1} - \wb_k) + \frac{L}{2}\norm{\wb_{k+1} - \wb_k}^2 \\
    \intertext{Substituting the gradient descent update $\wb_{k+1} - \wb_k = -\eta \nabla_{\wb} \Lupper(\wb_k)$:}
    \Lupper(\wb_{k+1}) &\le \Lupper(\wb_k) - \eta \norm{\nabla_{\wb} \Lupper(\wb_k)}^2 + \frac{L\eta^2}{2}\norm{\nabla_{\wb} \Lupper(\wb_k)}^2 \\
    &= \Lupper(\wb_k) - \eta \left(1 - \frac{L\eta}{2}\right) \norm{\nabla_{\wb} \Lupper(\wb_k)}^2.
\end{align*}
\end{proof}

\begin{proof}[Proof of Theorem 1]%
Let $\delta = \eta(1 - L\eta/2)$. Since $\eta \in (0, 2/L)$, we have $\delta > 0$. Rearranging the inequality from Lemma~\ref{lem:sufficient_decrease} gives:
\begin{equation*}
    \delta \norm{\nabla_{\wb} \Lupper(\wb_k)}^2 \le \Lupper(\wb_k) - \Lupper(\wb_{k+1}).
\end{equation*}
We now sum this inequality from $k=0$ to $T$ to form a telescoping series:
\begin{align*}
    \sum_{k=0}^{T} \delta \norm{\nabla_{\wb} \Lupper(\wb_k)}^2 &\le \sum_{k=0}^{T} (\Lupper(\wb_k) - \Lupper(\wb_{k+1})) \\
    &= (\Lupper(\wb_0) - \Lupper(\wb_1)) + (\Lupper(\wb_1) - \Lupper(\wb_2)) + \dots \\
    &\quad\quad{}+ (\Lupper(\wb_T) - \Lupper(\wb_{T+1})) \\
    &= \Lupper(\wb_0) - \Lupper(\wb_{T+1}).
\end{align*}
The objective function is bounded below by $\mathfrak{L}_{\inf} \geq 0$. Therefore, $\Lupper(\wb_{T+1}) \ge \mathfrak{L}_{\inf}$. This gives us:
\begin{equation*}
    \sum_{k=0}^{T} \delta \norm{\nabla_{\wb} \Lupper(\wb_k)}^2 \le \Lupper(\wb_0) - \mathfrak{L}_{\inf}.
\end{equation*}
As $T \to \infty$, the right-hand side is a finite constant. This implies that the infinite series of squared gradient norms is bounded:
\begin{equation*}
    \sum_{k=0}^{\infty} \norm{\nabla_{\wb} \Lupper(\wb_k)}^2 \le \frac{\Lupper(\wb_0) - \mathfrak{L}_{\inf}}{\delta} < \infty.
\end{equation*}
For an infinite series of non-negative terms to converge to a finite value, the terms themselves must converge to zero. Therefore, we must have:
\begin{equation*}
    \lim_{k \to \infty} \norm{\nabla_{\wb} \Lupper(\wb_k)}^2 = 0,
\end{equation*}
which implies that $\lim_{k \to \infty} \norm{\nabla_{\wb} \Lupper(\wb_k)} = 0$. This completes the proof that the algorithm converges to a stationary point.
\end{proof}

\section{Proofs for Projection Error Analysis}
\label{app:proofs_projection}

This appendix provides the foundational definitions and detailed proofs for the projection error analysis presented in Section 4.2.%

\subsection{Foundational Concepts}
\label{app:concepts}
\begin{definition}[Universal Kernel]
A continuous kernel \( k \) defined on a compact metric space \( (\mathcal{X}, d) \) is called a \textbf{universal kernel} if the Reproducing Kernel Hilbert Space (RKHS) \( \mathcal{H}_k \) induced by \( k \) is dense in the space of continuous functions \( C(\mathcal{X}) \) with respect to the uniform norm \( \|\cdot\|_{\infty} \).

Mathematically, this means that for any function \( g \in C(\mathcal{X}) \) and any \( \varepsilon > 0 \), there exists a function \( f \in \mathcal{H}_k \) such that:
\[
\sup_{x \in \mathcal{X}} |f(x) - g(x)| < \varepsilon.
\]

An equivalent way to state this is that the closure of \( \mathcal{H}_k \) under the uniform norm is \( C(\mathcal{X}) \):
\[
\overline{\mathcal{H}_k} = C(\mathcal{X}).
\]
\end{definition}

\begin{definition}
\label{def:projection_error}
Let $f \in \mathcal{L}^2(\Omega, \R)$ be a target function. The \textbf{projection error} of $f$ onto an Hilbert space $\mathcal{H} \subseteq \mathcal{L}^2(\Omega, \R)$ is defined as
\begin{equation}
    \mathrm{Err}(f, \mathcal{H}) := \inf_{g \in \mathcal{H}} \|f - g\|.
\end{equation}
If this infimum is attained by some $g^* \in \mathcal{H}$, then $g^*$ is the projection of $f$ onto $\mathcal{H}$.
\end{definition}

\begin{theorem}[Composition of Universal Kernels]
\label{thm:app_composite_universal}
If $k(z, z')$ is a universal kernel on a space $\mathcal{Z}$, and the mapping $h:\mathcal{X} \to \mathcal{Z}$ is continuous and sufficiently expressive (e.g., injective), then the composite kernel $k_h(x, x') := k(h(x), h(x'))$ is universal on $\mathcal{X}$.
\end{theorem}
\begin{remark}
    The universality of the composite kernel relies on the composition
theorem from \cite{micchelli2006universal}.
\end{remark}

\begin{theorem}[RFF Approximation]
\label{rff_stdaprox}
Let \(k: \mathbb{R}^m \times \mathbb{R}^m \to \mathbb{R}\) be a continuous, translation-invariant, positive definite kernel function, i.e.,
\[
k(x, x') = k(x - x') = \int_{\mathbb{R}^d} e^{i w^T (x - x')} d\mu(w),
\]
where \(\mu\) is a probability measure with compact support. Define the Random Fourier Feature (RFF) map $\phi_{\mathrm{RFF}}: \mathcal{X} \to \mathbb{R}^{2D}$ as
\begin{equation}
    \phi_{\mathrm{RFF}}(x) 
    := \sqrt{\frac{1}{D}}
    \begin{bmatrix}
        \cos(w_1^\top x + b_1)\\
        \cos(w_D^\top x + b_D)\\
        \vdots \\
         \sin(w_1^\top x + b_1) \\
        \sin(w_D^\top x + b_D)
    \end{bmatrix},
\end{equation}
where $w_i \overset{\mathrm{i.i.d.}}{\sim} \mu$ and $b_i \overset{\mathrm{i.i.d.}}{\sim} \mathrm{Uniform}[0, 2\pi]$, with $\{w_i\}$ independent of $\{b_i\}$.

Let \(\mathcal{H}_k\) be the RKHS corresponding to \(k\). For any \(f \in \mathcal{H}_k\) and any probability distribution \(\rho\):
\[
\lim_{D \to \infty} \inf_{\theta \in \mathbb{R}^D} \left\| f - \phi_{\mathrm{RFF}}^T \theta \right\|_{L^2(\rho)} = 0,
\]
i.e., the function space spanned by RFF features is dense in \(\mathcal{H}_k\).
\end{theorem}
\begin{remark}
    This result is a direct consequence of the uniform convergence of the RFF kernel approximation to the true kernel \cite{rahimi2007random}.
\end{remark}

\begin{theorem}[RFF Approximation for Composite Kernels]
\label{thm:app_rff_composite}
Let $k$ be a continuous, translation-invariant, positive definite kernel function on $\R^m$ and $h: \R^d \to \R^m$ be a continuous mapping. Let $\mathcal{H}_{k_h}$ be the RKHS of the composite kernel $k_h(x, x') = k(h(x), h(x'))$. The function space spanned by the composite RFF features, $\mathcal{H}_{\mathrm{RFF}}$ defined in Equation 8 of the main paper, %
is dense in $\mathcal{H}_{k_h}$ with respect to the $L^2(\rho)$ norm when $D \to \infty$.
\[
\lim_{D \to \infty} \inf_{\theta \in \mathbb{R}^D} \left\| f(x) - \psi(x)^\top \theta \right\|_{L^2(\rho)} = 0
\]
\end{theorem}

\begin{proof}
The proof connects the approximation properties in the base space \(\mathcal{H}_k\) to the composite space \(\mathcal{H}_{k_h}\) through the mapping \(h\).
The composite RKHS \(\mathcal{H}_{k_h}\) consists of functions formed by composing elements from the base RKHS \(\mathcal{H}_k\) with the mapping \(h\), that is, \(\mathcal{H}_{k_h} = \{ g(h(\cdot)) \mid g \in \mathcal{H}_k \}\). Thus, for any function \(f \in \mathcal{H}_{k_h}\), there exists a corresponding function \(g \in \mathcal{H}_k\) such that \(f(x) = g(h(x))\) for all \(x \in \mathbb{R}^d\).

To proceed, define a standard Random Fourier Feature (RFF) map for the base kernel \(k(z, z')\) on~\(\mathbb{R}^m\):
\[
\psi_D(z) := \sqrt{\frac{1}{D}} \begin{bmatrix} \cos(w_1^T z) \\ \vdots \\ \cos(w_D^T z) \\ \sin(w_1^T z) \\ \vdots \\ \sin(w_D^T z) \end{bmatrix},
\]
where \(w_i \overset{\mathrm{i.i.d.}}{\sim} \mu\) and \(b_i \overset{\mathrm{i.i.d.}}{\sim} \mathrm{Uniform}[0, 2\pi]\), with \(\{w_i\}\) independent of \(\{b_i\}\). The classic RFF approximation (Theorem \ref{rff_stdaprox}) guarantees that the linear span of these features is dense in \(\mathcal{H}_k\) with respect to the \(L^2\) norm under suitable measures. Let \(\rho_h\) be the pushforward probability measure of \(\rho\) under the map \(h\). Then, for the function \(g \in \mathcal{H}_k\),
\[
\lim_{D \to \infty} \inf_{\theta \in \mathbb{R}^{2D}} \left\| g - \psi_D(x)^T \theta \right\|_{L^2(\rho_h)} = 0.
\]

The goal is to show that \(f(x)\) can be approximated by a function of the form \(\psi(x)^\top \theta\). We analyze the squared \(L^2(\rho)\) norm of the error:
\[
\left\| f(x) - \psi(x)^\top \theta \right\|_{L^2(\rho)}^2 = \int_{\mathbb{R}^d} \left| f(x) - \psi(x)^\top \theta \right|^2 d\rho(x).
\]
Substituting \(f(x) = g(h(x))\) and \(\psi(x) =\psi_D(h(x))\) yields
\[
\int_{\mathbb{R}^d} \left| g(h(x)) - \psi_D(h(x))^T \theta \right|^2 \, d\rho(x).
\]
By the change of variables (or pushforward measure property), this integral equals
\[
\int_{\mathbb{R}^m} \left| g(z) - \psi_D(z)^T \theta \right|^2 \, d\rho_h(z) = \left\| g - \psi_D^T \theta \right\|_{L^2(\rho_h)}^2.
\]
From the density in \(L^2(\rho_h)\), this error can be made arbitrarily small as \(D \to \infty\) by choosing appropriate \(\theta\). Therefore, for any \(f \in \mathcal{H}_{k_h}\),
\[
\lim_{D \to \infty} \inf_{\theta \in \mathbb{R}^{2D}} \left\| f - \psi_D^T \theta \right\|_{L^2(\rho)} = 0,
\]
completing the proof.
\end{proof}

\begin{lemma}[Expressive Power of the RFF-Enhanced Space]
\label{lem:denseness}
When the number of the RFF features $D \to \infty$, the Feature Space $\mathcal{H}_{f}$ is a subset of the closure of the Composite RFF Function Space $\mathcal{H}_{\mathrm{RFF}}$ with respect to the $\mathcal{L}^2(\Omega,\R)$ norm.
\end{lemma}

\begin{proof}
Let \(g_1\) be an arbitrary function in \(\mathcal{H}_{f}\).
Since $g_1$ is a linear combination of the continuous functions in $h(x)$, $g_1$ is a continuous function on a compact set $\mathcal{X}$, i.e., $g_1 \in C(\mathcal{X})$. By Theorem~\ref{thm:app_composite_universal}, the composite kernel $k_h$ is universal. Thus, its RKHS $\mathcal{H}_{k_h}$ is dense in $C(\mathcal{X})$ under the uniform norm. This means that for any $\epsilon > 0$, there exists a function $f_{k_h} \in \mathcal{H}_{k_h}$ such that:
\begin{equation*}
    \norm{g_1 - f_{k_h}}_{\infty} = \sup_{x \in \mathcal{X}} |g_1(x) - f_{k_h}(x)| < \frac{\epsilon}{2}.
\end{equation*}

For any probability measure $\rho$, the $L^2(\rho)$ norm is bounded by the $L_\infty$ norm, that is:
\begin{align*}
    \norm{g_1 - f_{k_h}}_{L^2(\rho)}^2 &= \int_{\mathcal{X}} |g_1(x) - f_{k_h}(x)|^2 d\rho(x) \\
    &\le \int_{\mathcal{X}} \left(\sup_{z \in \mathcal{X}} |g_1(z) - f_{k_h}(z)|\right)^2 d\rho(x) \\
    &= \norm{g_1 - f_{k_h}}_{\infty}^2 \int_{\mathcal{X}} d\rho(x) = \norm{g_1 - f_{k_h}}_{\infty}^2.
\end{align*}
Thus we have $\norm{g_1 - f_{k_h}}_{L^2(\rho)} \le \norm{g_1 - f_{k_h}}_{\infty} < \frac{\epsilon}{2}$.

From Theorem~\ref{thm:app_rff_composite}, the composite RFF space $\mathcal{H}_{\mathrm{RFF}}$ is dense in the RKHS $\mathcal{H}_{k_h}$ under the $L^2(\rho)$ norm. Therefore, for our function $f_{k_h}$ from the previous step, there exists a function $f_{\mathrm{RFF}} \in \mathcal{H}_{\mathrm{RFF}}$ such that:
\begin{equation*}
    \norm{f_{k_h} - f_{\mathrm{RFF}}}_{L^2(\rho)} < \frac{\epsilon}{2}.
\end{equation*}

Combining the results using the triangle inequality for the $L^2$ norm:
\begin{align*}
    \norm{g_1 - f_{\mathrm{RFF}}}_{L^2(\rho)} &\le \norm{g_1 - f_{k_h}}_{L^2(\rho)} + \norm{f_{k_h} - f_{\mathrm{RFF}}}_{L^2(\rho)} < \frac{\epsilon}{2} + \frac{\epsilon}{2} = \epsilon.
\end{align*}
Since for any $g_1 \in \mathcal{H}_{f}$ and any $\epsilon > 0$, we have found an element $f_{\mathrm{RFF}} \in \mathcal{H}_{\mathrm{RFF}}$ that is $\epsilon$-close in the $L^2$ norm, we have proven that $\mathcal{H}_{f} \subseteq \overline{\mathcal{H}_{\mathrm{RFF}}}$.
\end{proof}

\subsection{Proof of Theorem 2 %
(Projection Error Comparison)}
\label{app:proof_projection_error}

We have shown in Lemma~\ref{lem:denseness} that the premise \(\mathcal{H}_{f} \subseteq \overline{\mathcal{H}_{\mathrm{RFF}}}\) holds. The proof now proceeds as follows.

For any function \( g \in L^2 \), define its \( L^2 \) approximation error with respect to \( f \) as \( E(g) := \|f - g\|_{L^2} \). The function \( E: L^2 \to \mathbb{R} \) is continuous, which follows from the reverse triangle inequality, establishing that the norm is a continuous function.

Let \( g_1 \) be an arbitrary element in \( \mathcal{H}_{f} \). Since \( \mathcal{H}_{f} \subseteq \overline{\mathcal{H}_{\mathrm{RFF}}}\), by the definition of closure, there exists a sequence of functions \( \{g_2^{(n)}\}_{n=1}^{\infty} \) in \( \mathcal{H}_{\mathrm{RFF}} \) such that \( g_2^{(n)} \to g_1 \) in the \( L^2 \) sense, that is, 
\[ \lim_{n \to \infty} \|g_2^{(n)} - g_1\|_{L^2} = 0. \]

Because the error function \( E(\cdot) \) is continuous, we can interchange the function with the limit:
\[ \lim_{n \to \infty} E(g_2^{(n)}) = E\left(\lim_{n \to \infty} g_2^{(n)}\right) = E(g_1). \]

For each \( n \), \( g_2^{(n)} \) is an element of \( \mathcal{H}_{\mathrm{RFF}} \), so its error \( E(g_2^{(n)}) \) must be greater than or equal to the infimum of errors over \( \mathcal{H}_{\mathrm{RFF}} \): 
\[ \inf_{g \in \mathcal{H}_{\mathrm{RFF}}} E(g) \le E(g_2^{(n)}). \]
This inequality holds for all \( n \). Taking the limit as \( n \to \infty \) on both sides yields
\[ \inf_{g \in \mathcal{H}_{\mathrm{RFF}}} E(g) \le \lim_{n \to \infty} E(g_2^{(n)}). \]

Substituting the continuity result gives 
\[ \inf_{g \in \mathcal{H}_{\mathrm{RFF}}} E(g) \le E(g_1). \]

The inequality holds for any arbitrary \( g_1 \in \mathcal{H}_{f} \). This implies that \( \inf_{g \in \mathcal{H}_{\mathrm{RFF}}} E(g) \) is a lower bound for the set of values \( \{E(g) \mid g \in \mathcal{H}_{f}\} \). By the definition of an infimum (greatest lower bound), this value must be less than or equal to the infimum of the set:
\[ \inf_{g \in \mathcal{H}_{\mathrm{RFF}}} E(g) \le \inf_{g \in \mathcal{H}_{f}} E(g). \]

\subsubsection{Proof of the universal approximation corollary}
\label{app:universal_approximation}

The following inequality holds:
    \[
        \| u - u_{\omega,\theta} \|^2 \leq \| u - \tilde{u}_{\omega,W} \|^2 + \| \tilde{u}_{\omega,W} - u_{\omega,\theta} \|^2.
    \]
    The universal approximation theorem by \cite{hornik1991approximation} guarantees that for any $\varepsilon > 0$, there exists $p$ and $\omega_*, W_*$ such that $\| u - \tilde{u}_{\omega_*,W} \|^2 \leq \varepsilon / 2$. Theorem 2 of the main paper %
    ensures that there exists $D$ and $\theta_*$ such that $\| \tilde{u}_{\omega_*,W_*} - u_{\omega_*,\theta_*} \|^2 \leq \varepsilon / 2$. Consequently, the norm between $u$ and $u_{\omega_*,\theta_*}$ is smaller than $\varepsilon$ and that concludes the proof.

\section{IFeF-PINN Hyperparameters}
\label{app:hyper_setting}
This section details the hyperparameters used by the proposed IFeF-PINN method in each experiment (Table~\ref{tab:hyper_setting}). Here, $D$ denotes the number of Fourier-enhanced features; $\sigma$ is the standard deviation of the sampled frequencies in random Fourier features (RFF); $\gamma$ is the regularization parameter defined in Eq.~\eqref{eq:regularized_solution}; and Pre-training indicates a warm-up stage where a vanilla PINN is trained for several thousand epochs to provide a good initialization for basis selection in IFeF-PINN.

As discussed in \cite{wang2021eigenvector}, the selection of $\sigma$ should align with the target function's frequency content. However, we fix $\sigma=1$ across all cases in our experiments and obtain accurate approximations. In future work, a more detailed analysis of $\sigma$ will be presented.

\begin{table}[ht]
  \centering
  \setlength{\tabcolsep}{10pt}
  \renewcommand{\arraystretch}{1.15}
  \begin{tabular}{l c c c c c}
    \toprule
    \textbf{Problem} & Pre-training & $D$ & $\sigma$ & $\lambda_{\mathrm{LL}}$ & $\gamma$ \\
    \midrule
    2D Helmholtz ($a_1=1,a_2=4$)      & No & 800 & 1 & 1e-2 & 1e-6 \\
    2D Helmholtz ($a_1=a_2=100$)      & No & 2400 &1 & 1e-7 & 1e-4 \\
    1D Convection ($\beta=50$)        & Yes & 800 &1 & 1e-2 &1e-7 \\
    1D Convection ($\beta=200$)       & Yes & 1600 &1 & 1e-2 & le-4/1e-7 \\
    Viscous Burgers ($\nu = \frac{0.01}{\pi} $)                & Yes & 800 &1 & 1e-1 & 0 \\
    \shortstack{Convection-Diffusion \\ ($k_\text{low} = 4\pi$, $k_\text{high} = 60\pi$)}           & Yes & 800 &1 & 1e-2 & 1e-7 \\
    \bottomrule
  \end{tabular}
  \caption{Hyperparameters setting for IFeF-PINN under each experiment}
  \label{tab:hyper_setting}
  \vspace*{-0.5cm}
\end{table}

\section{Experiment Setup}
\label{app:pdes_setup}

\subsection{PDEs Setup}
\label{app:pde_setup}
In this section, we provide detailed PDE settings used as our benchmarks.

\textbf{2D Helmholtz Equation.} The Helmholtz equation is an elliptic PDE that commonly arises in the study of wave propagation, acoustics, and electromagnetic fields. We consider the 2D Helmholtz equation as follows:
\begin{equation}
\label{eq:helmholtz2d}
\begin{aligned}
&\nabla^2 u + u = f, \quad  (x,y) \in \Omega, \\
&u(x,y) = 0, \quad (x,y) \in \partial \Omega,
\end{aligned}
\end{equation}
corresponding to a source term
\[
    f(x,y) = - \pi^2 \left( a_1^2 \sin(a_1 \pi x) \, \sin(a_2 \pi y) - a_2^2 \sin(a_1 \pi x) \sin(a_2 \pi y) \right) + \sin(a_1 \pi x) \sin(a_2 \pi y).
\]
The parameters $a_1$ and $a_2$ define the frequency of the analytic solution $ u(x,y) = \sin(a_1\pi x)\sin(a_2\pi y)  $. We will investigate the following different frequency cases:
\begin{itemize}
    \item low-frequency: $a_1=1, a_2=4, \Omega=[-1,1]\times[-1,1]$.
    \item high-frequency: $a_1=100, a_2=100, \Omega=[0,0.2]\times[0,0.2]$.
\end{itemize}

The high-frequency case uses a reduced domain to maintain computational tractability. By the Nyquist-Shannon sampling criterion \cite{shannon2006communication}, resolving such high-frequency oscillations over a larger domain would require prohibitively dense collocation. Despite the smaller domain, the configuration covers $10\times 10$ wavelengths, capturing the extreme oscillatory behavior.

\textbf{1D Convection Equation.} The Convection equation is a hyperbolic PDE that describes the movement of a substance through fluids. We consider the periodic boundary conditions system as follows:
\begin{equation}
\label{eq:convection}
\begin{aligned}
    &\frac{\partial u}{\partial t} + \beta\frac{\partial u}{\partial x} = 0 , \quad (t,x) \in [0,1]\times[0, 2\pi] , \\
    &u(x,0)=\sin x,\\
    &u(0,t) = u(2\pi, t).
\end{aligned}
\end{equation}
The closed-form solution is $u(x,t) = \sin(x - \beta t)$. We consider a low-frequency case $\beta=50$ and a high-frequency case $\beta=200$ on the same domain.

\textbf{1D Convection-Diffusion Equation.} The Convection–Diffusion equation is a parabolic PDE that models the combined effects of transport by fluid motion and spreading due to diffusion. We consider the multi-scale system with periodic boundary conditions as follows:
\begin{equation}
\label{eq:convection-diffusion}
\begin{aligned}
& \frac{\partial u}{\partial t} + c \frac{\partial u}{\partial x} = d \frac{\partial^2 u}{\partial x^2}, \quad (t, x) \in [0,1]\times[0,1], \\
& u(t, 0) = u(t, 1), \\
& \frac{\partial u}{\partial x}(t, 0) = \frac{\partial u}{\partial x}(t, 1), \\
& u(0, x) = A_{\text{low}} \sin(k_{\text{low}} x) + A_{\text{high}} \sin(k_{\text{high}} x).
\end{aligned}
\end{equation}
The analytic multi-scale solution is
\begin{equation*}
    u(t, x) = A_{\text{low}} e^{-d k_{\text{low}}^2 t} \sin\left(k_{\text{low}}(x - ct)\right) + A_{\text{high}} e^{-d k_{\text{high}}^2 t} \sin\left(k_{\text{high}}(x - ct)\right).
\end{equation*}
To set a multi-scale problem consists of both low- and high-frequency components, the parameters are chosen as follows:
\[c = 1, d=0.00005, A_{\text{low}} = 1, A_{\text{high}}=0.1, k_{\text{low}} = 4\pi, k_{\text{high}} = 60\pi. \]

\textbf{Viscous Burgers' Equation.} The Viscous Burgers' equation is a nonlinear parabolic PDE that models fluid motion by combining convection and diffusion effects. We consider the nonlinear system as follows:
\begin{equation}
\begin{aligned}
& \frac{\partial u}{\partial t} + u \frac{\partial u}{\partial x} =  \nu\frac{\partial^2 u}{\partial x^2}, \quad (t, x) \in [0, 1] \times [-1,1], \\
& u(0, x) = -\sin (\pi x),\\
& u(t, -1) = u(t, 1) = 0, \\
\end{aligned}
\end{equation}
where $\nu = \frac{0.01}{\pi}$.

\subsection{Spectrum Analysis Setup}
\label{app:spectrum}
Given frequencies $\kappa = \{f_i\}_{i=1}^{i=10} = \{1, 2, 5, 10, 30, 40, 50, 60, 70, 80\}$, where all amplitudes are chosen as $A_i = 1$, we consider the Convection equation in~\ref{eq:convection} with $\beta =1$ and initial condition as follows:
\begin{equation*}
    u(x,0)=\sum_{i=1}^{10} A_i \sin( 2 \pi f_i \ x).
\end{equation*}
The corresponding analytic solution is then given by
\begin{equation*}
    u(x,t) = \sum_{i=1}^{10} A_i \sin\big(2 \pi f_i \ (x - t)\big).
\end{equation*}
The problem domain is defined as $(t,x) \in [0,1] \times [0,1]$.
Our objective is to evaluate and compare the ability of Vanilla PINNs and Fourier-enhanced Features to capture all frequency components at $t=0$.

For the neural network architecture, we employ an 8-layer fully connected network with $\tanh$ activation functions and 64 neurons per layer. The training data consist of $201$ uniformly sampled points along the two spatial boundaries ($x=0$ and $x=1$) and at the final time ($t=1$). In addition, $201 \times 201$ collocation points are uniformly sampled within the interior domain to enforce the physics constraints.

To further design this experiment as an ablation study and demonstrate that the incorporation of Fourier-enhanced Features for basis extension is a necessary component of our proposed IFeF-PINN, we discard the iterative training procedure and retain only the basis extension step. Specifically, we first train the Vanilla PINN for 40,000 epochs using the Adam optimizer with a learning rate of $10^{-3}$. We then extend the basis with varying numbers of Fourier-enhanced Features, $D_j \in \{400, 800, 1600, 2400, 3200, 4000\}$, and then solve the lower-level problem defined in Equation 6 of the main paper. %
Moreover, to ensure a more rigorous analysis, we impose the relation $B_{D_i} \subset B_{D_j}$ whenever $D_j > D_i$, so that the RFF mapping matrices are nested. 

\subsection{Model Setup}
\label{app:model_setup}
This section details the model setup for all baselines. Unless otherwise specified, we use a multi-layer perceptron (MLP) whose depth and width are determined by the experimental setting. For the 2D Helmholtz case ($a_1=1,\,a_2=4$), we follow the network structure in \cite{barreau2025accuracy}; for the viscous Burgers' equation, we follow \cite{raissi2019physics}. For PINNsformer \citep{ZhaoEtAl24} and PIG \citep{KangEtAl25}, since they both have special network architectures, we adopt the original architecture. All experiments use the $\tanh$ activation function and the Adam optimizer with a learning rate of $10^{-3}$ for the network parameters.
For IFeF-PD, we adopt the Primal-Dual weight balancing strategy proposed in \cite{barreau2025accuracy}, and optimize the dynamic physics weight with the same setting for Adam at a learning rate of $10^{-4}$. The network architectures used in each experiment are summarized in Table~\ref{tab:NN_archi}.

\begin{table}[ht]
  \centering
  \setlength{\tabcolsep}{10pt}
  \renewcommand{\arraystretch}{1.15}
  \begin{tabular}{lccc}
    \toprule
    \textbf{Problem} & \textbf{Hidden layers} & \textbf{Hidden width}  \\
    \midrule
    2D Helmholtz ($a_1=1,a_2=4$)                 & 3 & [50,50,20] \\
    2D Helmholtz ($a_1=a_2=100$)                 & 6 & 64 \\
    1D Convection ($\beta=50$)                     & 6 & 64  \\
    1D Convection ($\beta=200$)                     & 6 & 64  \\
    Viscous Burgers ($\nu = \frac{0.01}{\pi} $)                                 & 8 & 20\\
     \shortstack{Convection-Diffusion \\ ($k_\text{low} = 4\pi$, $k_\text{high} = 60\pi$)}                      & 6 & 64  \\
    \bottomrule
  \end{tabular}
  \caption{Network architecture for all problems.}
  \label{tab:NN_archi} 
  \vspace*{-0.5cm}
\end{table}

\subsection{Dataset setup}
\label{app:data_setup}
In this section, we detail the dataset setup for each equation and experiment. For the 1D Convection equation, 2D Helmholtz equation (low-frequency), and Convection-Diffusion equation, we follow the setting and strategy of \cite{ZhaoEtAl24}. For the Viscous Burgers' Equation, we follow \cite{raissi2019physics}. The detailed settings are summarized in Table~\ref{tab:dataset_setting}.

\begin{table}[htbp]
  \centering
  \setlength{\tabcolsep}{10pt}
  \renewcommand{\arraystretch}{1.15}
  \begin{tabular}{l c c c}
    \toprule
    \textbf{Problem} & \textbf{Sampling} & \textbf{Boundary points} & \textbf{Physics points}  \\
    \midrule
    2D Helmholtz ($a_1=1,a_2=4$)    & Uniform      
      & $1000$ & $71\times71$  \\
    2D Helmholtz ($a_1=a_2=100$) & LHS         
      & $3000$ & $23000$ \\
       \addlinespace[2pt]
    1D Convection ($\beta=50$)      & Uniform    
      & \makecell{$303^{\dagger}$ \\ $153^{\ddagger}$} 
      & \makecell{$(101\times101)^{\dagger}$ \\ $(51\times51)^{\ddagger}$} \\
       \addlinespace[2pt]
    1D Convection ($\beta=200$)     & Uniform       
      & $303$ & $101\times101$ \\
     Viscous Burgers ($\nu = \frac{0.01}{\pi} $)             & LHS        
      & $100$ & $10000$ \\
     \shortstack{Convection-Diffusion \\ ($k_\text{low} = 4\pi$, $k_\text{high} = 60\pi$)}     & Uniform
      & $404$ & $101\times101$ \\
    \bottomrule
  \end{tabular}
  \caption{Dataset settings for each PDE problem.}
  \vspace{2pt}
  {\footnotesize Notes: $^{\dagger}$ Vanilla/NTK/PIG; $^{\ddagger}$ PINNsformer/IFeF.}
  \label{tab:dataset_setting}
  \vspace*{-0.5cm}
\end{table}

\section{Experimental Results}
\label{app:additional_results}
In this section, we present the true solutions, model predictions, and absolute error maps for all baselines considered in our numerical experiments. Results for the viscous Burgers' equation, the low- and high-frequency convection equations, and the multi-scale convection-diffusion equation are shown in separate figures.
For clarity, each figure contains three panels: (i) the true solution, (ii) the model prediction, and (iii) the absolute error on a log10 scale. 

\begin{figure}
    \centering
    \begin{subfigure}{0.4\linewidth} 
        \centering
        \includegraphics[width=\linewidth]{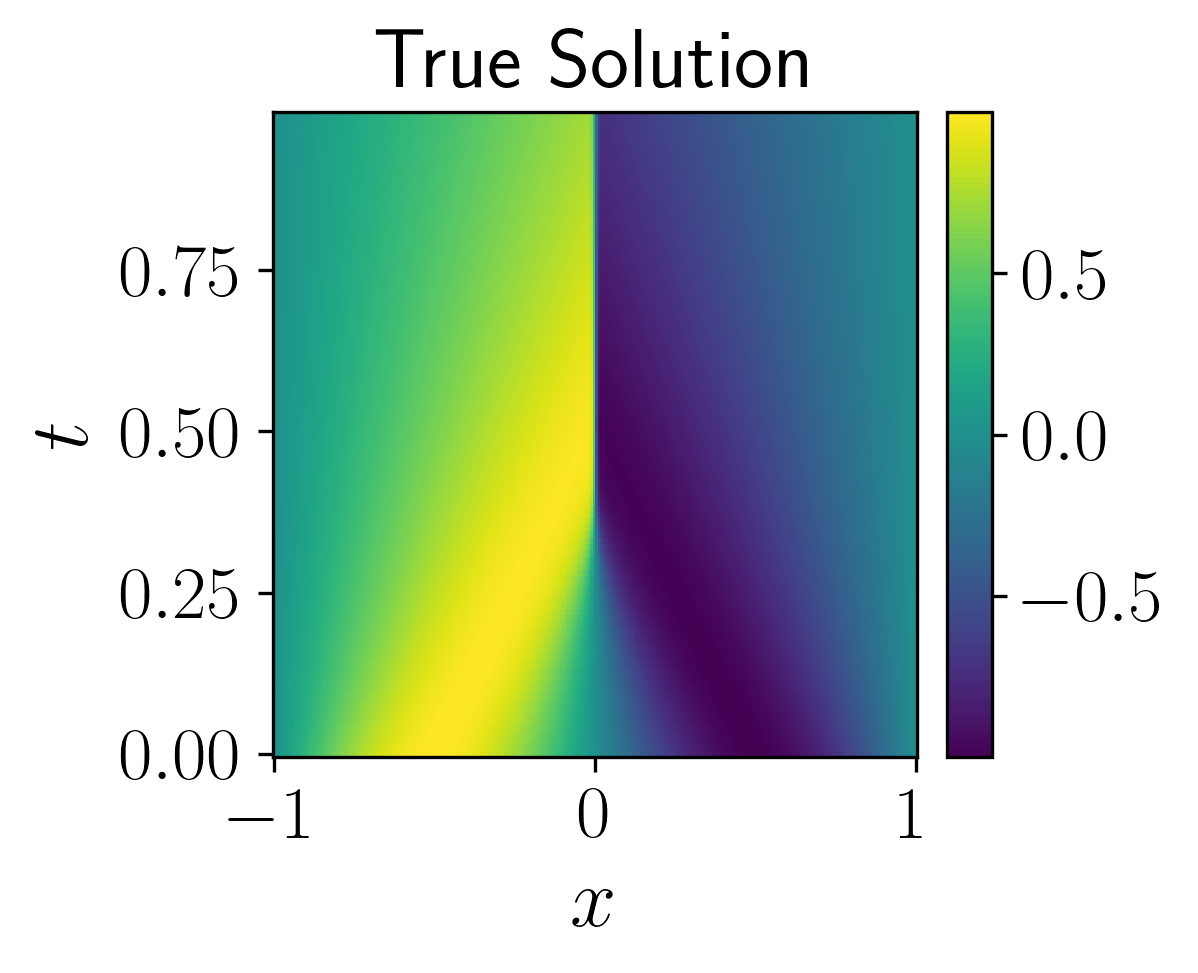}
        \caption{Exact solution}
    \end{subfigure}
    \\[0.3cm] 
    \begin{subfigure}{0.99\linewidth}
        \centering
        \includegraphics[width=\linewidth]{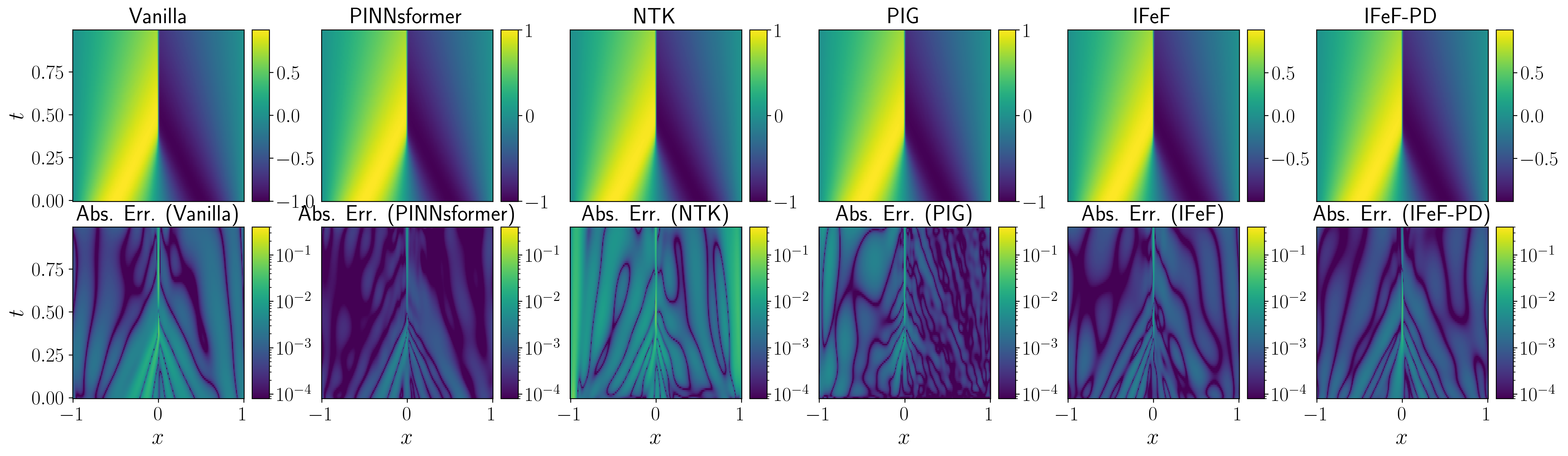}
        \caption{Prediction solution (top) and absolute error on a log10 scale (bottom) of baseline methods}
    \end{subfigure}
    \caption{True solution, prediction and absolute error of baseline methods for viscous Burgers' equation }
    \label{fig:burgers}
    \vspace*{-0.5cm}
\end{figure}

\begin{figure}
    \centering
    \begin{subfigure}{0.4\linewidth} 
        \centering
        \includegraphics[width=\linewidth]{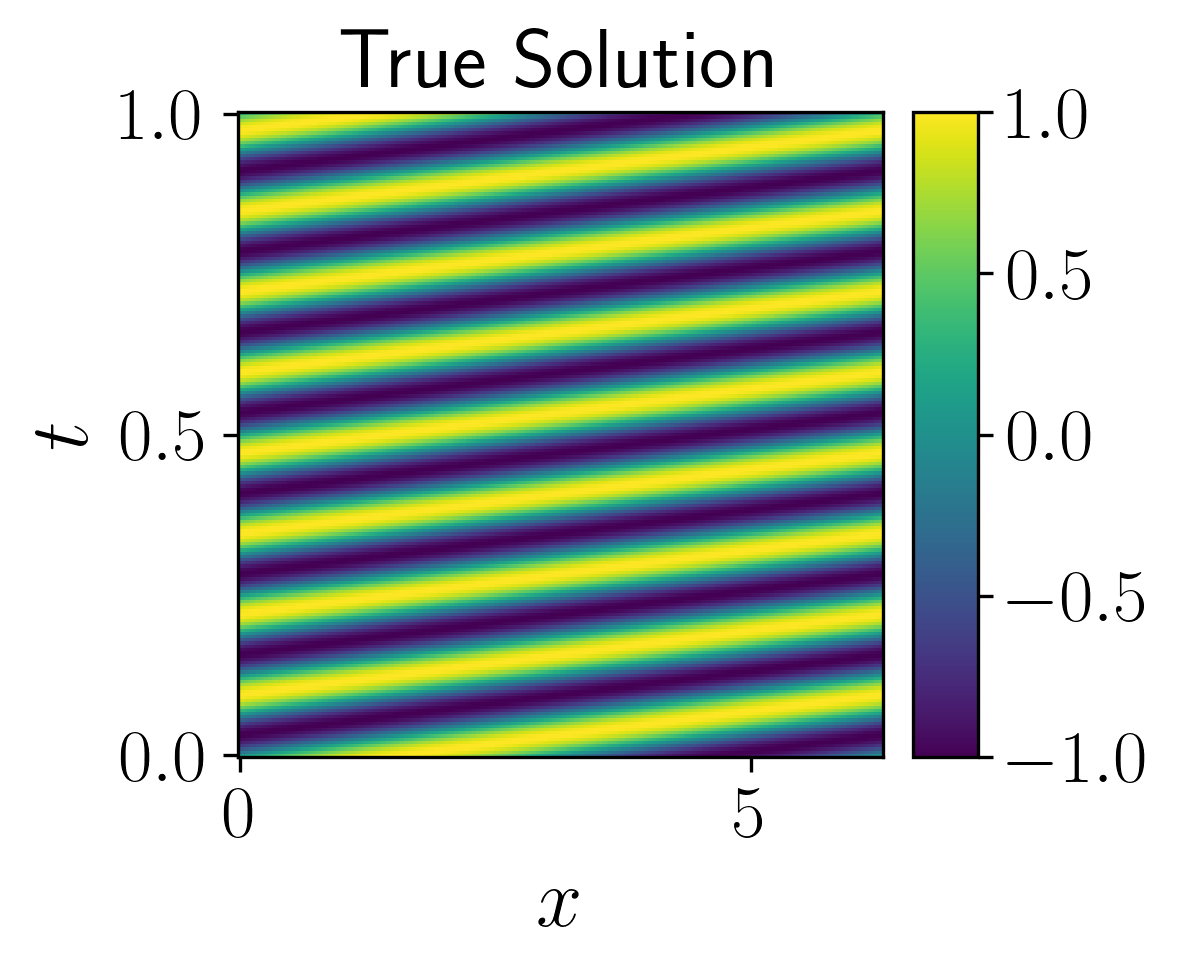}
        \caption{Exact solution}
    \end{subfigure}
    \\[0.3cm] 
    \begin{subfigure}{0.99\linewidth}
        \centering
        \includegraphics[width=\linewidth]{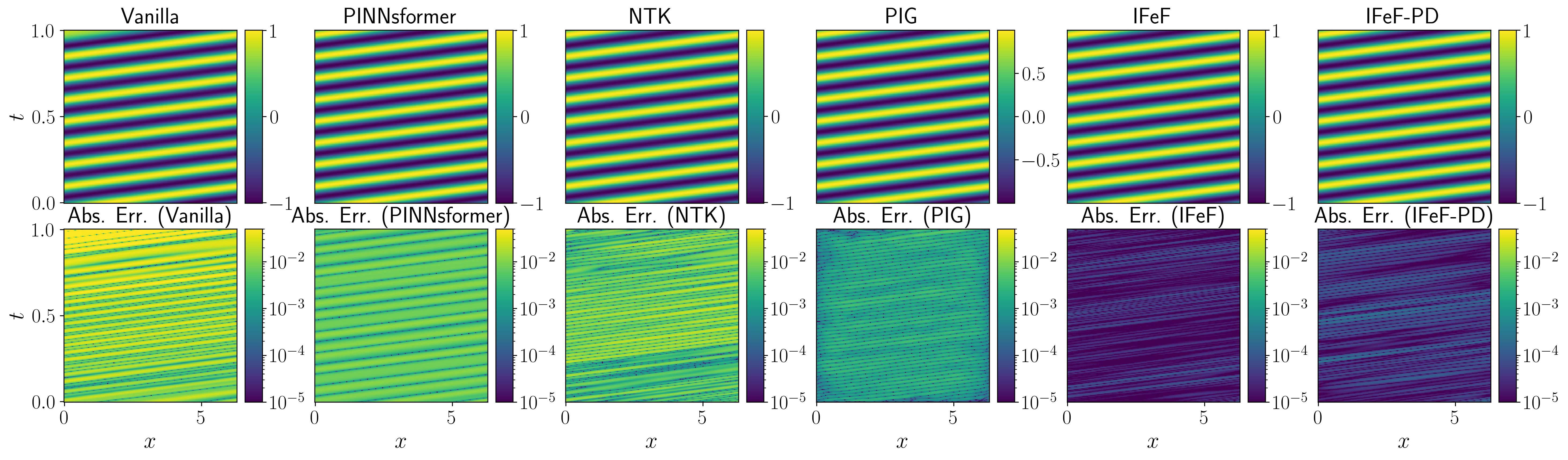}
        \caption{Prediction solution (top) and absolute error on a log10 scale (bottom) of baseline methods}
    \end{subfigure}
    \caption{True solution, prediction, and absolute error of baseline methods for low-frequency convection equation }
    \label{fig:convection_lowfi}
    \vspace*{-0.5cm}
\end{figure}

\begin{figure}
    \centering
    \begin{subfigure}{0.4\linewidth} 
        \centering
        \includegraphics[width=\linewidth]{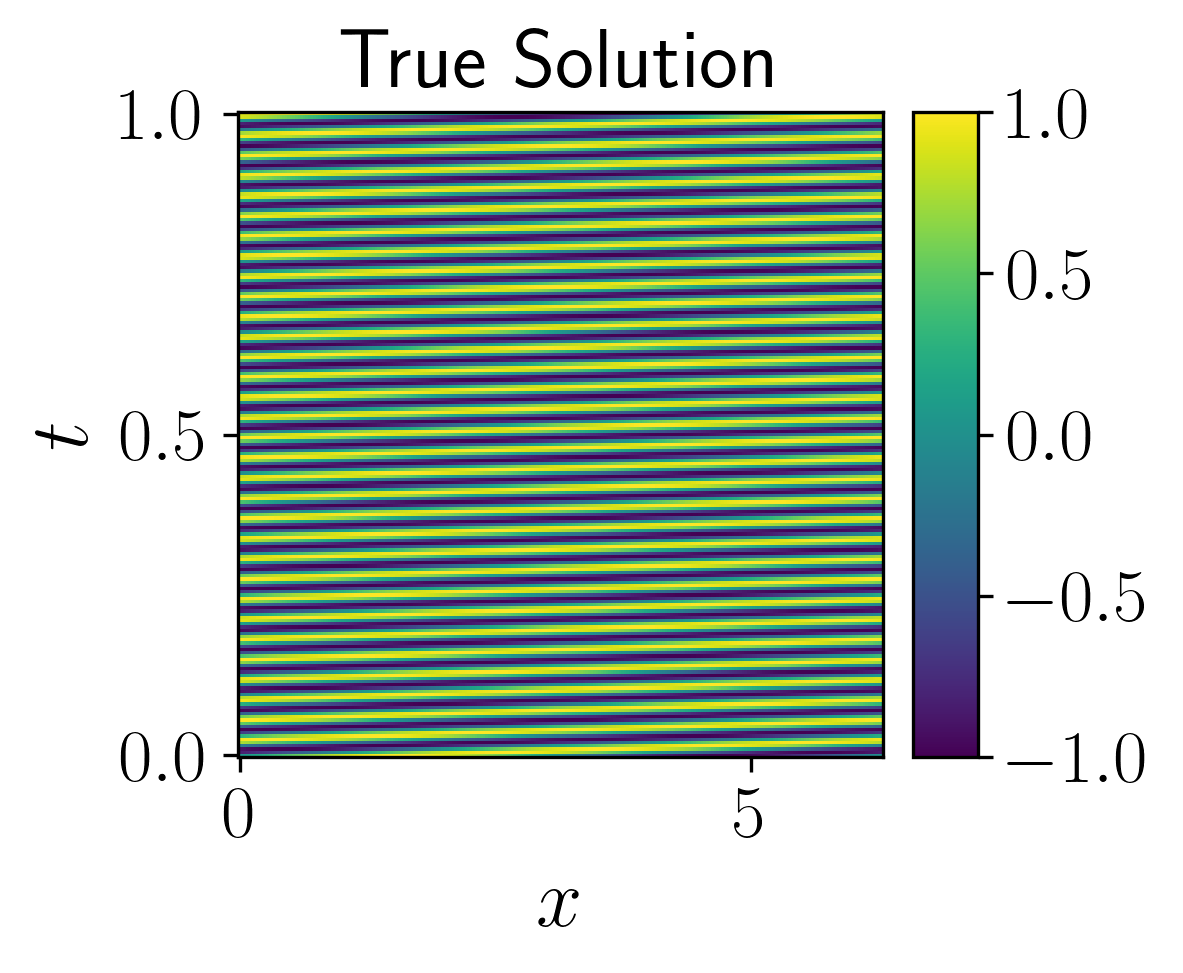}
        \caption{Exact solution}
    \end{subfigure}
    \\[0.3cm] 
    \begin{subfigure}{0.99\linewidth}
        \centering
        \includegraphics[width= \linewidth]{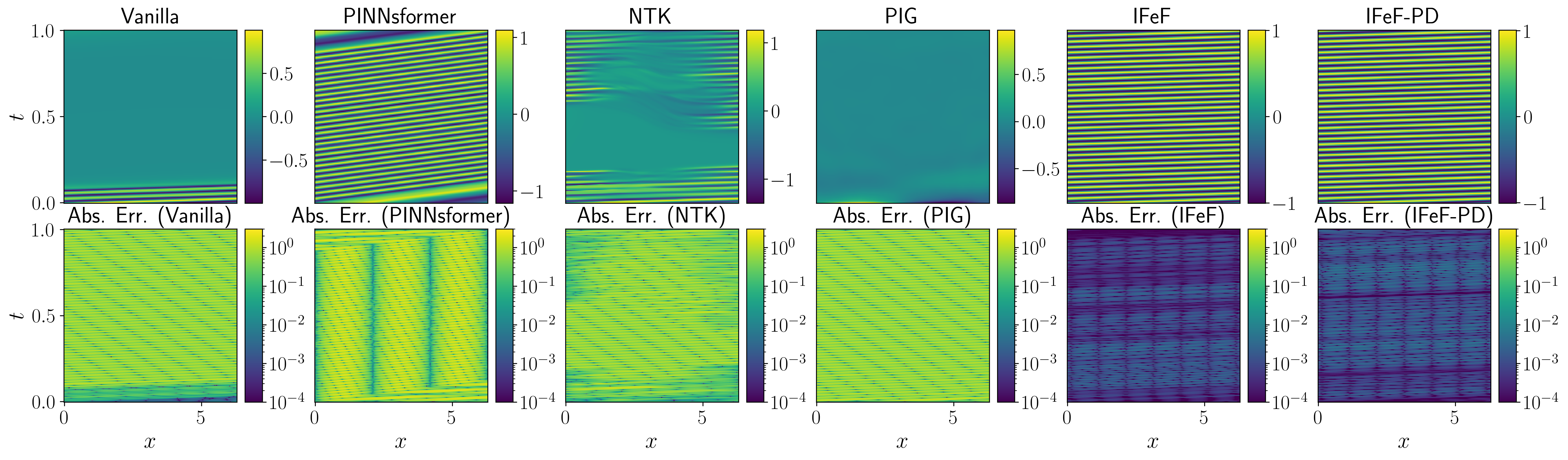}
        \caption{Prediction solution (top) and absolute error on a log10 scale (bottom) of baseline methods}
    \end{subfigure}
    \caption{True solution, prediction, and absolute error of baseline methods for high-frequency convection equation }
    \label{fig:convection_hifi}
    \vspace*{-0.5cm}
\end{figure}

\begin{figure}
    \centering
    \begin{subfigure}{0.4\linewidth} 
        \centering
        \includegraphics[width=\linewidth]{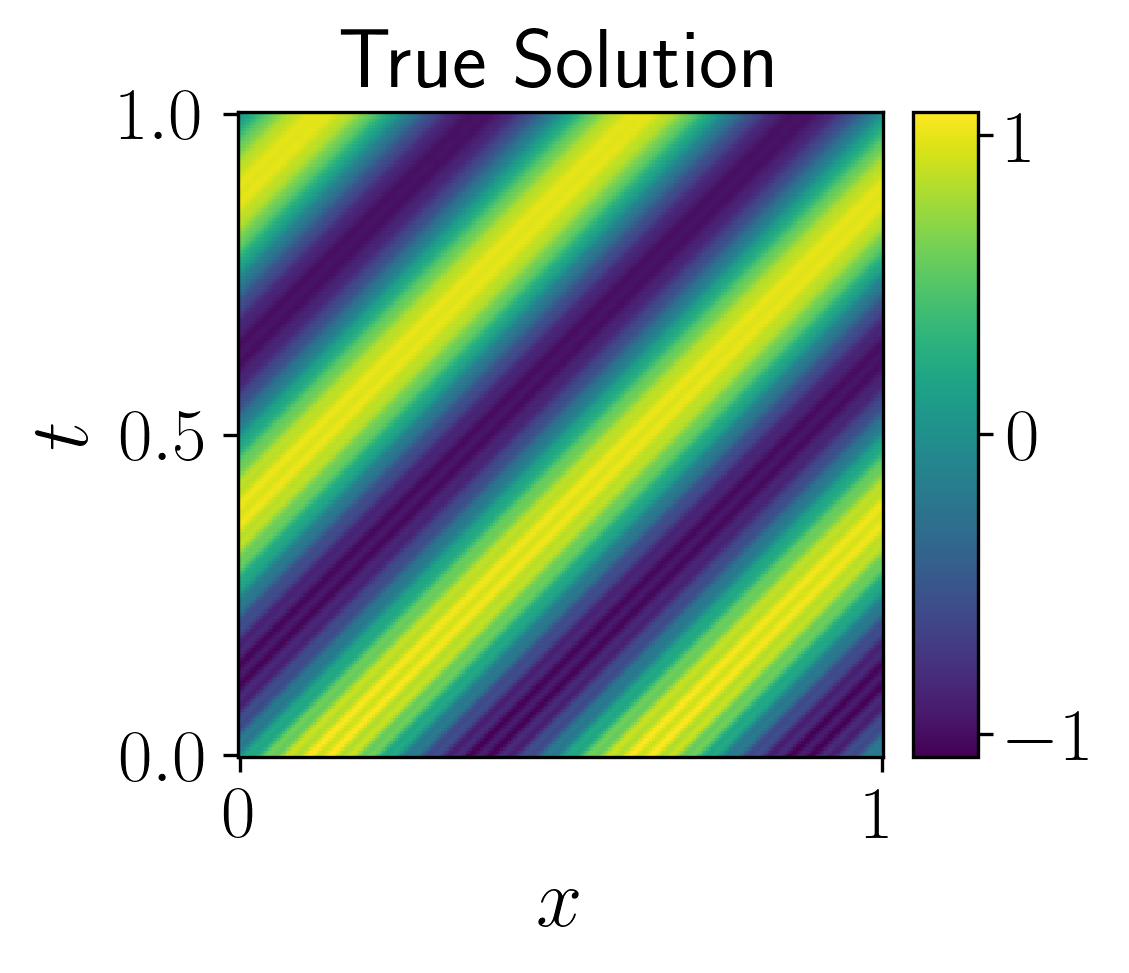}
        \caption{Exact solution}
    \end{subfigure}
    \\[0.3cm] 
    \begin{subfigure}{0.99\linewidth}
        \centering
        \includegraphics[width=\linewidth]{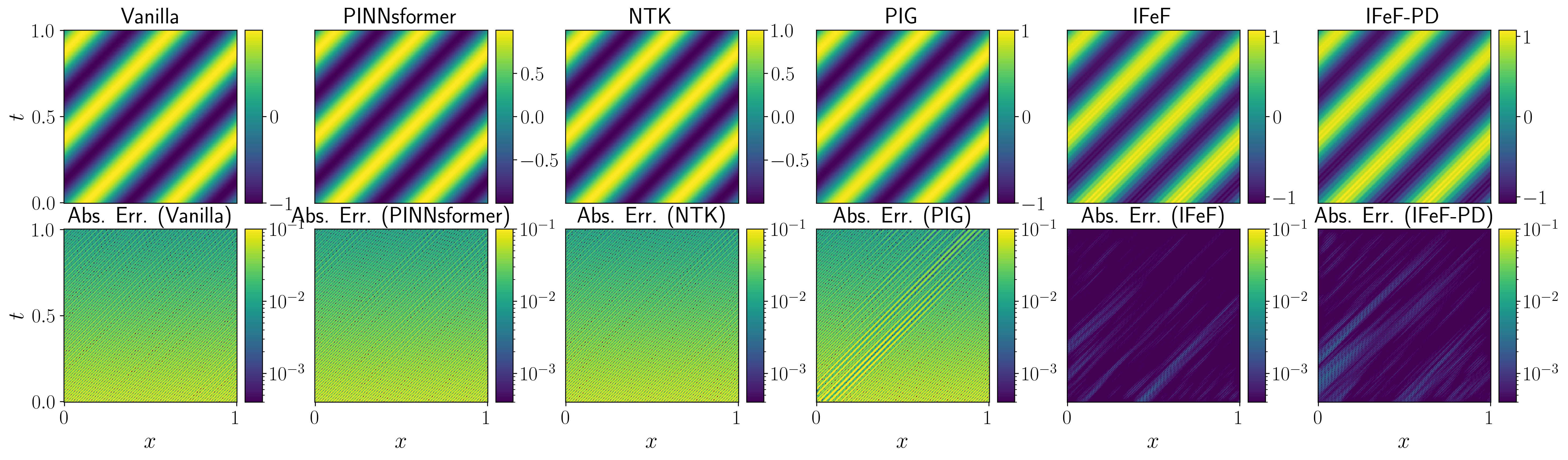}
        \caption{Prediction solution (top) and absolute error on a log10 scale (bottom) of baseline methods}
    \end{subfigure}
    \caption{True solution, prediction, and absolute error of baseline methods for multi-scale convection-diffusion equation }
    \label{fig:diffusion_multi}
    \vspace*{-0.5cm}
\end{figure}

The case presented in Figure~\ref{fig:diffusion_multi} is particularly interesting. We observe that the analytical solution is the sum of the two frequencies ($20$ and $60$). If we zoom in on the error plot, it is possible to see that for all other methods than IFeF, the high-frequency component is not caught. Despite visually similar plots, the error can be quite large. This phenomenon does not appear with IFeF-PINN.

\subsection{Comparison with Input-Space RFF}
\label{app:Baseline_MFF}
We compare IFeF-PINN with the Multi-scale      Fourier Features (MFF) method proposed by~\cite{wang2021eigenvector}, which applies multiple RFF mappings derived in Equation 4 of the main paper to the input layer of the neural network to mitigate spectral bias. 

Table~\ref{tab:mff_compare} summarizes the relative $L^2$-error and standard deviation across the low-frequency, high-frequency and multi-scale benchmarks. IFeF-PINN consistently outperforms MFF in approximation performance.

\begin{table}[htbp]
  \centering
  \setlength{\tabcolsep}{10pt}
  \renewcommand{\arraystretch}{1.15}
  \begin{tabular}{l c c c }
    \toprule
    \textbf{Baseline} & \shortstack{Convection \\ ($\beta=50$)}  & \shortstack{Convection \\ ($\beta=200$)}  & \shortstack{Convection-Diffusion \\ ($k_\text{low} = 4\pi$, $k_\text{high} = 60\pi$)}  \\
    \midrule
     MFF      &$2.14 \times 10^{-2} (4.42\times 10^{-3})$ & $3.50 \times 10^{-1} (2.27\times 10^{-1})$ & $5.21 \times 10^{-2} (4.21\times 10^{-4})$  \\
      IFeF      &$7.0 \times 10^{-5} (1.6\times10^{-3}) $ & $2.7 \times 10^{-3}(1.0\times10^{-3}) $   & $9.0 \times 10^{-4}(3.0\times10^{-4}) $  \\
     IFeF-PD      &$9.0 \times 10^{-5} (5.0\times10^{-4}) $ & $2.5 \times 10^{-3}(5.0\times10^{-4}) $   & $1.0 \times 10^{-3}(2.0\times10^{-4}) $  \\
    \bottomrule
  \end{tabular}
  \caption{Average relative $L^2$-error with corresponding standard deviation across 3 benchmarks for IFeF-PINN, IFeF-PD and MFF.}
  \label{tab:mff_compare}
  \vspace*{-0.25cm}
\end{table}

\subsection{Comparison with computational cost}
We provide a comparison of the computational costs for the 5 linear benchmarks among IFeF-PINN, vanilla PINNs, and the SOTA baseline PIG proposed by~\cite{KangEtAl25}. Tables~\ref{tab:ifef_computation} and~\ref{tab:baseline_computation} summarize the average training time per epoch, total training time, and memory usage for the three methods. To better analyze the computational cost in IFeF-PINN, we decompose the per-epoch training time into the upper-level and lower-level components.

\begin{table}[htbp]
  \centering
  \setlength{\tabcolsep}{6pt}
\renewcommand{\arraystretch}{1.15}
  \begin{tabular}{l c c c c}
    \toprule
 \textbf{IFeF-PINN}  & Per upper (s)  & Per lower (s) & Total time (s) & \ Memory (GB)  \\
    \midrule
    \shortstack{2D Helmholtz \\ ($a_1=1,a_2=4$)}    &  $0.015$  & $0.003$ & $448$  & $1.41$  \\
    \shortstack{2D Helmholtz \\ ($a_1=a_2=100$)} & $0.154$  & $0.042$ & $1960$ & $18.5$   \\
       \addlinespace[2pt]
    \shortstack{1D Convection \\ ($\beta=50$)}      & $0.024$   & $0.003$ &$ 108$ & $4.80$  \\
       \addlinespace[2pt]
    \shortstack{1D Convection \\ ($\beta=200$)}     & $0.051$  & $0.010$ & $610$  &  $5.85$  \\
     \shortstack{Convection-Diffusion \\ ($k_\text{low} = 4\pi$, $k_\text{high} = 60\pi$)} & $0.052$  & $0.003$  & $110$ & $ 5.26$  \\
    \bottomrule
  \end{tabular}
  \caption{Average training time per epoch for upper- and lower-level, total training time and memory usage for IFeF-PINN among 5 linear benchmarks.}
    \label{tab:ifef_computation}
  \vspace*{-0.25cm}
\end{table}

\begin{table}[htbp]
  \centering
  \setlength{\tabcolsep}{5pt}
  \renewcommand{\arraystretch}{1.15}
  \begin{tabular}{l ccc ccc}
    \toprule
    & \multicolumn{3}{c}{\textbf{Vanilla PINN}} & \multicolumn{3}{c}{\textbf{PIG}} \\
    \cmidrule(lr){2-4} \cmidrule(lr){5-7}
    \textbf{Problem} & \makecell{Per epoch\\(s)} & \makecell{Total\\(s)} & \makecell{Memory\\(GB)} & \makecell{Per epoch\\(s)} & \makecell{Total\\(s)} & \makecell{Memory\\(GB)} \\
    \midrule
    2D Helmholtz ($a_1=1,a_2=4$) & $0.003$ & $116$ & $0.2$ & $0.62$ & $248$ & $6.8$ \\
    2D Helmholtz ($a_1=a_2=100$) & $0.006$ & - & $0.55$ & $1.70$ & - & $20.6$ \\
    \addlinespace[2pt]
    1D Convection ($\beta=50$) & $0.002$ & $18$ & $0.10$ & $1.13$ & $565$ & $5.8$ \\
    \addlinespace[2pt]
    1D Convection ($\beta=200$) & $0.002$ & - & $0.10$ & $1.63$ & - & $14.7$ \\
    \shortstack{Convection-Diffusion \\ ($k_\text{low} = 4\pi$, $k_\text{high} = 60\pi$)} & $0.003$ & $27$ & $0.19$ & $0.85$ & $43$ & $9.8$ \\
    \bottomrule
  \end{tabular}
  \caption{Average training time per epoch, total training time and memory usage for Vanilla PINNs and PIG among 5 linear benchmarks. A dash '-' denotes that the method failed to achieve a meaningful approximation for the corresponding equation and is therefore excluded from the total training time.}
  \label{tab:baseline_computation}
  \vspace*{-0.25cm}
\end{table}

The results demonstrate that vanilla PINN achieves the fastest training but the poorest accuracy. PIG with its default L-BFGS optimizer converges in fewer epochs but incurs the highest memory cost due to the evaluation of numerous learnable Gaussian bases at collocation points. IFeF-PINN demonstrates lower memory usage than PIG while maintaining acceptable training time. Notably, IFeF-PINN's training is dominated by upper-level basis learning, while the lower-level QP solving is highly efficient.

\end{document}